\PassOptionsToPackage{table}{xcolor}
\documentclass{article}

\usepackage{iclr2027_conference}
\newif\ifarxivpreprint
\arxivpreprinttrue
\ifarxivpreprint
  \iclrfinalcopy
\fi

\usepackage{times}
\usepackage{microtype}
\usepackage{graphicx}
\usepackage{float}
\usepackage{subcaption}
\usepackage{booktabs}
\usepackage{array}
\usepackage{multirow}
\usepackage{amsmath,amssymb,amsfonts,amsthm}
\usepackage{mathtools}
\usepackage{algorithm}
\usepackage{algpseudocode}
\usepackage{enumitem}
\usepackage{xcolor}

\definecolor{oursbg}{RGB}{248,245,255}
\definecolor{ModernPurple}{RGB}{109,40,217}
\definecolor{ModernPurpleDark}{RGB}{76,29,149}
\definecolor{moderngreen}{RGB}{32,121,89}
\definecolor{modernred}{RGB}{185,28,28}
\definecolor{ChangeOrange}{RGB}{230,112,0}
\definecolor{TheoryGreen}{RGB}{34,85,73}
\definecolor{SafeGreen}{RGB}{34,85,73}
\definecolor{SafeTint}{RGB}{232,248,241}
\definecolor{BestGreen}{RGB}{4,120,87}
\definecolor{BestTint}{RGB}{232,248,241}
\definecolor{TableHead}{RGB}{250,247,255}
\definecolor{TableRule}{RGB}{221,214,254}

\usepackage[
    colorlinks=true,
    linkcolor=TheoryGreen,
    urlcolor=TheoryGreen,
    citecolor=TheoryGreen
]{hyperref}
\usepackage{cleveref}

\DeclareRobustCommand{\method}{\textsc{VICS}}
\DeclareRobustCommand{\vicsg}{\textsc{VICS-G}}
\DeclareRobustCommand{\vicss}{\textsc{VICS-S}}
\DeclareRobustCommand{\vicsr}{\textsc{VICS-R}}
\DeclareRobustCommand{\rcd}{\textsc{RCD}}
\DeclareRobustCommand{\vla}{\textsc{VLA}}
\newcommand{\policy}{\pi_\theta}
\newcommand{\priorlaw}{p_0}
\newcommand{\hist}{h_t}
\newcommand{\obs}{o_t}
\newcommand{\instr}{x}
\newcommand{\traj}{\tau}
\newcommand{\cand}{\mathcal{A}_K}
\newcommand{\zfeas}{Z_{\mathrm{feas}}}
\newcommand{\zmag}{\overline Z_{\mathrm{feas}}}
\newcommand{\safeset}{\mathcal{G}_x}
\newcommand{\truesupp}{\mathcal{F}_t}
\newcommand{\retained}{\mathcal{G}_t}
\newcommand{\unresolved}{\mathcal{R}_t}
\newcommand{\viableuniv}{\mathcal{V}_t}
\newcommand{\E}{\mathbb{E}}
\newcommand{\Prob}{\mathbb{P}}
\newcommand{\KL}{\mathrm{KL}}
\newcommand{\cost}{C}
\newcommand{\reward}{R}
\newcommand{\score}{S}
\newcommand{\ffilter}{\mathcal{F}_{\mathrm{pred}}}
\newcommand{\sfilter}{\mathcal{S}_{\mathrm{pred}}}
\newcommand{\bcand}{\mathcal{B}_t}

\newcommand{\rankloss}{\Gamma^{\mathrm{cov}}_t}
\newcommand{\rejectloss}{\Gamma^{\mathrm{rej}}_t}
\newcommand{\gateloss}{\Gamma^{\mathrm{gate}}_t}

\newcommand{\trueS}{S^{\star}}
\newcommand{\hatS}{\widehat S}

\DeclareMathOperator*{\argmax}{arg\,max}
\DeclareMathOperator*{\argmin}{arg\,min}

\newtheorem{definition}{Definition}
\newtheorem{proposition}{Proposition}
\newtheorem{theorem}{Theorem}
\newtheorem{corollary}{Corollary}

\title{A Safe Action Is Not Enough: Feasible-Future Decoding for Vision-Language-Action Policies}

\author{%
\begin{minipage}[t]{\dimexpr\textwidth-2\tabcolsep\relax}
\centering\normalfont\normalsize\bfseries
\begin{tabular}[t]{@{}*{3}{>{\centering\arraybackslash}p{0.333\linewidth}@{}}}
Tu Nguyen\textsuperscript{1,}\thanks{Corresponding author:
\href{mailto:tu.nguyen@huawei.com}{\texttt{tu.nguyen@huawei.com}}}
& Matthieu Zimmer\textsuperscript{2}
& Vu Anh Vu\textsuperscript{1}\\[5pt]
Ziyi Wang\textsuperscript{1,3}
& Jannik Hammel Nielsen\textsuperscript{1}
& Xuebing Zhou\textsuperscript{1}\\[5pt]
\multicolumn{3}{c}{Haitham Bou Ammar\textsuperscript{2,4}}
\end{tabular}\endgraf\vspace{10pt}
{\normalfont\small
\textsuperscript{1}Huawei Heisenberg Research Center\qquad
\textsuperscript{2}Huawei Noah's Ark Lab\\[3pt]
\textsuperscript{3}TU Berlin\qquad
\textsuperscript{4}UCL Center for AI}
\end{minipage}%
}

\hypersetup{
    pdftitle={A Safe Action Is Not Enough: Feasible-Future Decoding for Vision-Language-Action Policies}
}
\ifarxivpreprint
  \hypersetup{pdfauthor={Tu Nguyen, Matthieu Zimmer, Vu Anh Vu, Ziyi Wang, Jannik Hammel Nielsen, Xuebing Zhou, Haitham Bou Ammar}}
\else
  \hypersetup{pdfauthor={}}
\fi

\begin{document}

\maketitle
\ifarxivpreprint
  \lhead{Preprint}
  \raggedbottom 
\else
  \lhead{Under review as a conference paper at ICLR 2027}
\fi

\begin{abstract}

A safe action is not necessarily a viable one. A frozen vision-language-action (VLA) policy can favor a locally admissible move that leaves no policy-supported route to safe task completion. We call this the \emph{feasibility--likelihood} gap: likelihood ranks the next move, while feasibility depends on the futures it leaves open.

To bring those futures into the decision, we derive the exact next-block marginal of the history-conditioned policy--environment trajectory law restricted to safe task completion. The derivation reveals a candidate-dependent \emph{feasible-future mass}: its support records whether safe completion remains possible under the frozen continuation process, while its magnitude measures how much weighted safe-completion mass remains. Since exact evaluation is impractical online, we develop a selective finite-candidate approximation and establish conditions for recovering the best retained viable candidate.

Our alarm-triggered, training-free reranker \vicsg{} lowers mean cumulative safety cost by $1.9\%$--$57.5\%$ across six Safety-CHORES settings while remaining within $2.5$ percentage points of policy sampling in success and $0.82$ steps in mean episode length. Our approach offers a promising and practical path toward safer task completion, grounded in an exact policy-relative target yet requiring neither policy retraining nor online rollouts.

\end{abstract}

\begin{figure}[H]
\centering
\includegraphics[width=\linewidth]{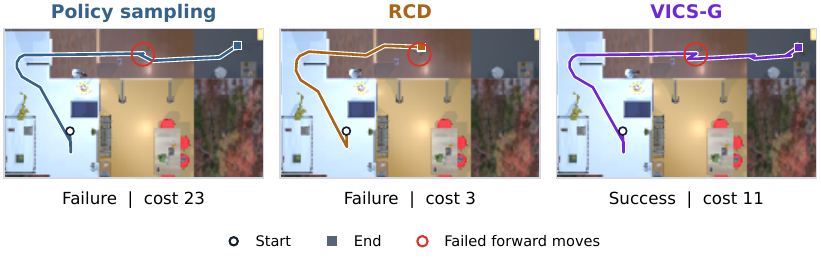}
\caption{\textbf{When progress stalls, find a way forward.}
In a selected ObjectNav episode (\emph{navigate to an apple}), policy sampling
fails at cost $23$, while the local-robustness decoder \rcd{} ends early at cost $3$.
\textbf{\vicsg{} breaks the cycle of failed retries and reaches the goal at cost $11$}.
Circles mark the region of failed forward moves. All three incur safety
cost; the full trajectories and intervention trace appear in
Appendix~\ref{app:qualitative-episodes}.}
\label{fig:frontpage-recovery}
\end{figure}
\ifarxivpreprint
\clearpage
\fi

\section{Introduction}
\label{sec:introduction}

An action can look safe in the moment and still steer the agent into a dead end. Under a frozen vision-language-action (\vla{}) policy, even a likely, locally admissible move may leave no policy-supported route to safe task completion. We call this the \emph{feasibility--likelihood} gap. The navigation example in \Cref{fig:frontpage-recovery} shows why completion belongs in the safety story: reducing cost by ending an episode early can leave the task unfinished.

VLA policies bring broad semantic and motor priors to robot control, learned from large-scale vision--language pretraining and robot demonstrations. Systems such as RT-2~\citep{brohan2023rt2}, PaLM-E~\citep{driess2023palme}, OpenVLA~\citep{kim2024openvla}, Octo~\citep{octo2024}, and $\pi_0$~\citep{black2024pi0} illustrate this breadth. Given an observation $\obs$, instruction $\instr$, and interaction history $\hist$, the policy defines
$$
    a_{t:t+H-1}\sim\policy(\cdot\mid\obs,\instr,\hist).
$$
These priors express what the policy has learned to favor, but do not by themselves establish whether it can finish safely. Deployment-time constrained decoding gives us a way to revisit action selection as execution unfolds, while keeping the underlying policy fixed~\citep{zhang2025safevla,kapoor2025safedec,hsu2021safety}.

Looking ahead in robotics means reasoning about a future that the action alone does not determine. In autoregressive language generation, appending a candidate token fixes the next symbolic prefix~\citep{ji2025inferenceguard}. A robot's next state also depends on the world in which it acts:
$$
    s_{t+H}\sim P(\cdot\mid s_t,a_{t:t+H-1}).
$$
Under partial observability and execution uncertainty, the same observable history and action block can lead to different successor states, with different reachability, visibility, and collision geometry. A safe-completion target must therefore account for both future policy choices and environment dynamics.

We begin with a law over complete trajectories and work back to the next decision. Weighting the joint policy--environment trajectory law by reward and cost, restricting it to safe task completion, and marginalizing over future continuations yields an exact next-block marginal. This marginal brings together the action prior, local reward--cost terms, and the weighted mass of safe, successful continuations---the \emph{feasible-future mass}.

This mass is relative to the policy whose futures we are considering. When it is zero, the frozen policy--environment process assigns no probability to safe completion after that action, even if another controller could still succeed. And even when a safe route remains, the policy may place little mass on the continuations that follow it. We therefore distinguish \emph{support}, which records whether safe completion remains possible, from \emph{magnitude}, which measures how much weighted continuation mass is preserved.

Exact evaluation is generally impractical online, but a decoder may still recover the preferred action from a much smaller comparison. Our selective finite-candidate approximation combines policy likelihood, local robustness, and a bounded continuation term. A gate first decides whether the sampled action is worth revisiting; only then does the decoder rerank a small candidate set. The gate buys efficiency at a price: a dead end it accepts never reaches the reranker. We account for this risk alongside errors in candidate availability, support handling, and approximate scoring, deriving conditions under which the best retained viable candidate can still be recovered.

These ideas lead to \method{} (\emph{Viability-Corrected Sampling}), a training-free reranker that draws on the current observation and interaction history. Across six Safety-CHORES settings~\citep{zhang2025safevla}, \vicsg{}, its generic rollout-free configuration, consistently lowers observed safety cost while keeping task completion and episode length close to policy sampling. Measuring all three together~\citep{fan2026safevlabench} keeps the practical aim in focus: reducing harm while the robot continues toward its goal.

We also ask how this picture changes when execution becomes less predictable, or when the decoder can afford a glimpse ahead. On safety-aligned Fetch, execution feedback yields a favorable observed safety--completion trade-off under perturbed actions, while a matched ablation supports preservation of task success within the preset tolerance. The theory invites us to let sampled futures shape the next action. Even in simulation, however, those futures are costly to generate and only as reliable as the dynamics model behind them. We therefore study \vicsr{}, a modest rollout augmentation of \vicsg{} guided by safe-completion estimates under a limited simulation budget. On Fetch, it improves observed safe success and further reduces cost (\Cref{sec:rollout-study}). Together, these studies test how observed transitions and sampled continuations can guide selective intervention.

Our contributions are threefold:
\begin{enumerate}[
leftmargin=*,
itemsep=2pt,
topsep=2pt,
parsep=0pt,
partopsep=0pt
]
\item
\textbf{An exact next-block target for policy-relative safe completion.}
We derive the next-action marginal of a prior-preserving trajectory law restricted to safe completion. Feasible-future mass connects the current action to the safe continuations it leaves open: support captures whether such continuations remain, while magnitude captures their weighted mass.

\item
\textbf{A theory of selective approximation and its limits.}
We separate intervention, coverage, support, and ranking errors, derive a recovery margin for the best retained viable action, and show when finite continuation rollouts preserve the ideal winner.

\item
\textbf{A practical route to safer execution with frozen policies.}
\vicsg{} combines selective, training-free intervention with consistent observed cost reductions and near-policy task completion across six Safety-CHORES settings. Studies of execution feedback and rollout augmentation test the value of evidence from realized actions and sampled continuations.
\end{enumerate}

\section{Local Safety and Policy-Supported Viability}
\label{sec:preliminaries}

We consider a frozen policy $\policy$ over $H$-action blocks
$a\equiv a_{t:t+H-1}$, conditioned on the observable history
$\mathcal H_t=(\obs,\instr,\hist)$. Executing a block changes the
latent state according to $s_{t+H}\sim P(\cdot\mid s_t,a)$.
Together with the transition kernel $P$ and observation process
$\Omega$, the policy induces a history-conditioned law
$\priorlaw(\traj\mid\mathcal H_t)$ over future trajectories.
Let $\safeset$ denote those trajectories that complete $\instr$
while satisfying the designated hard safety budget.

At inference time, the decoder receives a finite policy-supported
candidate set $\cand=\{a^1,\ldots,a^{N_t}\}$ and returns one of its
members. Here $K$ denotes the proposal budget; retaining the policy
sample alongside top-$K$ actions gives $N_t=|\cand|\le K+1$.

A local robustness decoder combines the policy's preference with
the predicted safety margin of the next block:
\begin{equation}
\label{eq:local-robustness-score}
\score_{\mathrm{loc}}(a)
=
\log \policy(a\mid\mathcal H_t)
+
\lambda_c\rho_C(\widehat\traj_{t:t+H}),
\end{equation}
where $\widehat\traj_{t:t+H}$ is the predicted local trace,
$\lambda_c\ge 0$ weights local robustness, and larger $\rho_C$
indicates greater predicted constraint margin. Likelihood expresses
what the policy favors; local robustness assesses the predicted
next block. Neither alone tells us whether the action leaves a
supported route to safe completion.

\begin{definition}[Feasibility--likelihood gap]
A feasibility--likelihood gap occurs at $\mathcal H_t$ if there
exist $a,b\in\cand$ such that
\[
\log\policy(a\mid\mathcal H_t)
>
\log\policy(b\mid\mathcal H_t),
\]
but
\[
\Prob_{\priorlaw}[\traj\in\safeset\mid\mathcal H_t,a]
<
\Prob_{\priorlaw}[\traj\in\safeset\mid\mathcal H_t,b].
\]
Each probability measures safe task completion after choosing the
candidate, averaging over latent-state uncertainty, environment
outcomes, and subsequent actions of the frozen policy.
The policy favors $a$, yet its own continuation process gives $b$
a higher probability of safe completion.
\end{definition}

Two failure modes make this gap concrete. An \emph{unsafe shortcut}
violates a local constraint. A \emph{safe dead end} passes the local
check but leaves zero probability of safe completion under the
frozen continuation process. A dead end for this policy may still
allow safe task completion under another controller.
The question is how the safe futures each action leaves open
should shape the choice made now.

\section{From Trajectories to Feasible Futures}
\label{sec:method}

Safe completion is a property of a whole trajectory, even though the decoder
must choose one action block at a time. The reference construction follows
standard KL-regularized control and control-as-inference machinery
\citep{todorov2006linearly,toussaint2009robot,levine2018reinforcement}.
We derive the ideal next-action score before tracing how selective decoding
over a finite candidate set can depart from it.
Appendix~\ref{app:reference-target} gives the variational formulation and
regularity conditions, while Appendix~\ref{app:controlled-diagnostics}
verifies the principal identities by exact finite-model enumeration.

\subsection{A feasible-future identity}
\label{sec:suffix-theory}

Safe completion enters the next decision through the prior-preserving reference law
\begin{equation}
\label{eq:global-target}
q^\star(\traj\mid\mathcal H_t)
=
\frac{1}{Z(\mathcal H_t)}
\priorlaw(\traj\mid\mathcal H_t)
\mathbf 1\{\traj\in\safeset\}
\exp\!\left(
\beta\reward_\instr(\traj)-\lambda\cost(\traj)
\right),
\end{equation}
provided $0<Z(\mathcal H_t)<\infty$. The indicator restricts the
reference law to safe task completion; the exponential weight expresses
reward--cost preferences among those trajectories. This is a reference
distribution, not a claim that the deployed decoder solves a global control
problem.

Write a future realization as $(a,\xi)$, where $a\equiv a_{t:t+H-1}$ is the
candidate command block. The residual suffix $\xi$ includes stochastic
execution of $a$ and subsequent latent states, observations, policy actions,
and outcomes. Suffix reward and cost notation suppresses conditioning on
$(\mathcal H_t,a)$. Only history-and-action-measurable terms enter
$\reward_{\mathrm{loc}},\cost_{\mathrm{loc}}$ and are pulled outside the
expectation; random block rewards and collision costs remain inside the
continuation expectation.

\begin{theorem}[Feasible futures in the blockwise marginal]
\label{thm:blockwise-suffix-mass}
Assume
\begin{equation}
\label{eq:prior-factorization}
\priorlaw(a,\xi\mid\mathcal H_t)
=
\policy(a\mid\mathcal H_t)
\,\priorlaw(\xi\mid\mathcal H_t,a),
\end{equation}
and let reward and cost decompose as
\begin{align}
\reward_\instr(a,\xi)
&=
\reward_{\mathrm{loc}}(\mathcal H_t,a)
+
\reward_{\mathrm{suf}}(\xi,\instr),\\
\cost(a,\xi)
&=
\cost_{\mathrm{loc}}(\mathcal H_t,a)
+
\cost_{\mathrm{suf}}(\xi).
\end{align}
For cumulative constraints, the remaining budget is included in
$\mathcal H_t$; Appendix~\ref{app:formal-robot-decoding} gives the
augmented-state form. Then the next-block marginal of
\Cref{eq:global-target} is
\begin{equation}
\label{eq:blockwise-target-main}
q^\star(a\mid\mathcal H_t)
\propto
\policy(a\mid\mathcal H_t)
\exp\!\left(
\beta\reward_{\mathrm{loc}}(\mathcal H_t,a)
-
\lambda\cost_{\mathrm{loc}}(\mathcal H_t,a)
\right)
\zfeas(\mathcal H_t,a),
\end{equation}
where
\begin{equation}
\label{eq:suffix-mass}
\zfeas(\mathcal H_t,a)
=
\E_{\xi\sim\priorlaw(\cdot\mid\mathcal H_t,a)}
\left[
\mathbf 1\{(a,\xi)\in\safeset\}
\exp\!\left(
\beta\reward_{\mathrm{suf}}(\xi,\instr)
-
\lambda\cost_{\mathrm{suf}}(\xi)
\right)
\right].
\end{equation}
\end{theorem}

\paragraph{Interpretation.}
After the future is marginalized out, its influence remains as
feasible-future mass alongside the frozen action prior and local
reward--cost contribution. This mass integrates latent-state uncertainty, stochastic execution, future
observations, and future policy actions under the frozen continuation
process. When $\beta=\lambda=0$, this factor is simply the conditional
probability of safe task completion after the candidate. In particular, $\zfeas(\mathcal H_t,a)=0$ means that this process
assigns no mass to safe task completion after $a$; it does not mean that every
controller would fail from the resulting state. Taking the logarithm of the
unnormalized marginal gives the ideal score defined later in
\Cref{eq:ideal-total-score}.

\paragraph{Support and magnitude.}
Safe completion can remain possible even when the policy is unlikely to
find its way there. We therefore separate two questions: does safe completion have positive
probability under the frozen continuation process, and, when it does, how much
weighted continuation mass remains? Define
\begin{equation}
\label{eq:support-magnitude}
\chi_{\mathrm{feas}}^\star(\mathcal H_t,a)
=
\mathbf 1\{\zfeas(\mathcal H_t,a)>0\},
\qquad
\zmag^\star(\mathcal H_t,a)
=
\begin{cases}
\zfeas(\mathcal H_t,a), & \zfeas(\mathcal H_t,a)>0,\\
1, & \zfeas(\mathcal H_t,a)=0,
\end{cases}
\end{equation}
so that
\begin{equation}
\zfeas(\mathcal H_t,a)
=
\chi_{\mathrm{feas}}^\star(\mathcal H_t,a)
\,\zmag^\star(\mathcal H_t,a).
\end{equation}
The value $1$ off support is a neutral convention. Support is categorical;
magnitude is positive and graded on the support. A finite score can make a
suspected dead end unlikely, but it cannot reproduce the literal zero, or the
$-\infty$ log score, of an action outside the ideal support.

\paragraph{Safe dead ends.}
Two candidates can have equal likelihood and identical local reward--cost
terms while one has positive feasible-future mass and the other has none.
Every local-only score is then indifferent, whereas the ideal blockwise rule
assigns the dead end zero mass. Appendix~\ref{app:proof-safe-dead-end} states
this separation formally.

\paragraph{An exact finite witness.}

\Cref{fig:exact-finite-witness} follows four candidates: an unsafe shortcut
$a_U$, a locally admissible dead end $a_D$, and feasible actions $a_L,a_H$
with lower and higher continuation magnitude. Each factor changes the winner;
Appendix~\ref{app:controlled-diagnostics} gives the exact construction.

\begin{figure}[htbp]
    \centering
    \includegraphics[width=0.94\linewidth]{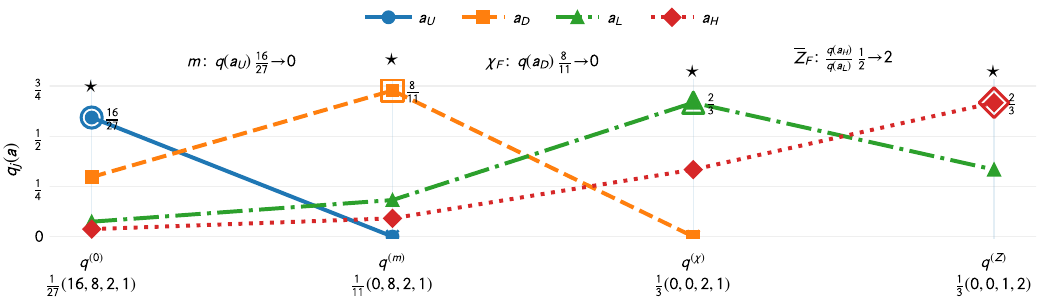}
    \caption{\textbf{Local safety is not enough.}
    Local admissibility removes the unsafe shortcut $a_U$, but only future
    support removes the safe dead end $a_D$; among viable actions, feasible-future
    magnitude then favors $a_H$ over $a_L$. The sequence
    $a_U\!\to a_D\!\to a_L\!\to a_H$ shows successive ideal scoring factors,
    not physical time.}
    \label{fig:exact-finite-witness}
\end{figure}

The witness also exposes a limit of selective decoding: a gate that accepts
$a_D$ leaves the viable alternatives unexamined.

\subsection{Selective finite candidates: four approximation interfaces}
\label{sec:finite-candidate-theory}

The ideal score tells us what to prefer; in practice, however, the decoder
sees only a few candidates and an imperfect picture of their futures.
We trace the path to execution through four approximation interfaces:
\textbf{intervention}, whether authorized reranking determines the executed
action; \textbf{coverage}, which viable alternatives survive proposal and
admission; \textbf{support}, whether viable actions are rejected or dead ends
remain eligible; and \textbf{ranking}, how accurately the score orders viable
survivors. We first analyze recovery within the reranking branch, then account
for retention of the policy sample.

Let $a_{\mathrm{pol}}\in\cand$ be the policy sample and $a_0$ the
highest-likelihood candidate; they need not coincide. On intervention, the
admitted set combines a likelihood trust region with shared safety and
task-feasibility checks:
\begin{equation}
\label{eq:practical-candidate-set}
\bcand
=
\cand\cap\mathcal T_\tau(a_0)\cap\sfilter\cap\ffilter.
\end{equation}
Here $\mathcal T_\tau(a_0)$ limits the policy log-probability gap
from $a_0$, and $\sfilter,\ffilter$ are predicted-safety and task-feasibility
filters (Appendix~\ref{app:decoder-procedure}). These checks do not determine
the true viable set.

The true policy-relative viable subset of $\bcand$ is
\begin{equation}
\label{eq:true-feasible-support}
\truesupp
=
\{a\in\bcand:\ \zfeas(\mathcal H_t,a)>0\}.
\end{equation}
Let $\mathcal D_t\subseteq\bcand$ be the candidates rejected by a practical
support rule using its available pre-action information. This set need not
equal the true dead-end set
$\bcand\setminus\truesupp
 =\{a\in\bcand:\zfeas(\mathcal H_t,a)=0\}$. Define
\begin{equation}
\label{eq:partial-support-sets}
\widehat{\mathcal F}_t=\bcand\setminus\mathcal D_t,
\qquad
\retained=\truesupp\cap\widehat{\mathcal F}_t,
\qquad
\unresolved=\widehat{\mathcal F}_t\setminus\truesupp.
\end{equation}
Thus $\widehat{\mathcal F}_t$ is the set available for reranking:
$\retained$ contains its viable members, and $\unresolved$ contains its
remaining dead ends. The decoder constructs $\widehat{\mathcal F}_t$ from
pre-action rules; the split into $\retained$ and $\unresolved$ is used only
for analysis. We do not assume that the decoder knows which surviving
candidates truly support safe completion.

Taking the logarithm of the unnormalized marginal in
\Cref{thm:blockwise-suffix-mass}, define the extended-real ideal score
\begin{equation}
\label{eq:ideal-total-score}
\trueS(a)
=
\log\policy(a\mid\mathcal H_t)
+
\beta\reward_{\mathrm{loc}}(\mathcal H_t,a)
-
\lambda\cost_{\mathrm{loc}}(\mathcal H_t,a)
+
\log\zfeas(\mathcal H_t,a),
\end{equation}
with $\log0=-\infty$. Maximizing $\trueS$ therefore gives the exact
next-block MAP rule induced by \Cref{thm:blockwise-suffix-mass}. Let
$\hatS(a)$ be any finite practical score on $\widehat{\mathcal F}_t$.
Whenever $\retained\neq\emptyset$, define the practical support-separation
margin
\begin{equation}
\label{eq:practical-support-separation}
\Delta_t^{\mathrm{sup}}
=
\max_{g\in\retained}\hatS(g)
-
\max_{r\in\unresolved}\hatS(r).
\end{equation}
The second maximum is $-\infty$ when $\unresolved=\emptyset$.
Thus $\Delta_t^{\mathrm{sup}}>0$ ensures that a retained viable action outranks
every remaining dead end. Rejection of viable actions is accounted for
separately by the loss $\rejectloss$ below.

The next theorem bounds the loss relative to the best retained viable
candidate and identifies when that candidate is recovered exactly.
Only score differences determine the ranking, so a candidate-independent
offset does not affect the recovery guarantee.

\begin{theorem}[Selective finite-candidate recovery under partial support]
\label{thm:candidate-ranking-consistency}
Let $\viableuniv$ be a finite policy-supported viable comparison set
containing $\truesupp$ and, whenever $a_{\mathrm{pol}}$ is viable,
the retained policy sample $a_{\mathrm{pol}}$.
Because all decision sets are finite, the maxima below are attained. Assume
$\retained\neq\emptyset$, $\Delta_t^{\mathrm{sup}}>0$, and
\begin{equation}
\label{eq:pairwise-total-score-error}
\left|
\bigl(\hatS(a)-\hatS(b)\bigr)
-
\bigl(\trueS(a)-\trueS(b)\bigr)
\right|
\le\varepsilon
\qquad
\forall a,b\in\retained.
\end{equation}
Let
\[
\widehat a=\argmax_{a\in\widehat{\mathcal F}_t}\hatS(a),
\quad
a_{\mathcal G}^\star=\argmax_{a\in\retained}\trueS(a),
\quad
a_{\mathcal F}^\star=\argmax_{a\in\truesupp}\trueS(a),
\quad
a_{\mathcal V}^\star=\argmax_{a\in\viableuniv}\trueS(a),
\]
under a common deterministic tie rule. Then $\widehat a\in\retained$ and
\begin{equation}
\label{eq:ranking-regret-bound}
\trueS(a_{\mathcal G}^\star)-\trueS(\widehat a)
\le\varepsilon.
\end{equation}
If the ideal margin of $a_{\mathcal G}^\star$ over the second-best action in
$\retained$ exceeds $\varepsilon$, then
$\widehat a=a_{\mathcal G}^\star$; take the margin as $+\infty$ when
$\retained$ is a singleton. Moreover,
\begin{equation}
\label{eq:coverage-ranking-decomposition}
\trueS(a_{\mathcal V}^\star)-\trueS(\widehat a)
\le
\underbrace{\trueS(a_{\mathcal V}^\star)-\trueS(a_{\mathcal F}^\star)}_{\rankloss\ \text{(proposal/admission coverage)}}
+
\underbrace{\trueS(a_{\mathcal F}^\star)-\trueS(a_{\mathcal G}^\star)}_{\rejectloss\ \text{(false rejection)}}
+
\varepsilon.
\end{equation}
\end{theorem}

The three reference actions make the decomposition concrete:
$a_{\mathcal V}^\star$ is best in the viable comparison set,
$a_{\mathcal F}^\star$ is best after proposal and admission, and
$a_{\mathcal G}^\star$ is best after support rejection. The nested sets
$\retained\subseteq\truesupp\subseteq\viableuniv$ make the first two losses
nonnegative: they measure opportunities lost before ranking begins.
The remaining ranking loss within $\retained$ is bounded by $\varepsilon$.

The bound concerns the ideal action score at the current history.
Support separation, $\Delta_t^{\mathrm{sup}}>0$, ensures that the practical
winner is viable even when some dead ends remain eligible. The gate and
final authorization separately determine whether that winner is executed.

\paragraph{Selective intervention.}
For \vicsg{}, let $g_t\in\{0,1\}$ indicate whether the initial gate invokes
reranking. Set $j_t=1$ exactly when $g_t=1$, the surviving set is nonempty,
and final replacement authorization accepts $\widehat a$; otherwise set
$j_t=0$. The deployed action is
\begin{equation}
\label{eq:selective-deployed-action}
a_t^{\mathrm{dep}}
=
\begin{cases}
\widehat a, & j_t=1,\\
a_{\mathrm{pol}}, & j_t=0.
\end{cases}
\end{equation}
For nonempty $\viableuniv$, define
$\mathcal L_t^{\mathrm{dep}}
 =\trueS(a_{\mathcal V}^\star)-\trueS(a_t^{\mathrm{dep}})$ and
\begin{equation}
\label{eq:gate-loss}
\gateloss
=
\trueS(a_{\mathcal V}^\star)-\trueS(a_{\mathrm{pol}})\in[0,\infty].
\end{equation}
The selective procedure gives
\begin{align}
\mathcal L_t^{\mathrm{dep}}
&=
\begin{cases}
\gateloss, & j_t=0,\\
\trueS(a_{\mathcal V}^\star)-\trueS(\widehat a), & j_t=1,
\end{cases}
\label{eq:selective-regret-decomposition}\\
\mathcal L_t^{\mathrm{dep}}
&\le
\rankloss+\rejectloss+\varepsilon,
\qquad j_t=1.
\label{eq:selective-reranking-bound}
\end{align}
The first line is an identity; the bound applies only when $j_t=1$ and the
assumptions of \Cref{thm:candidate-ranking-consistency} hold. The $j_t=0$ branch covers initial
gate acceptance, empty-set fallback, and final authorization veto. Its loss
$\gateloss$ can be infinite: if
$\zfeas(\mathcal H_t,a_{\mathrm{pol}})=0$, then
$\trueS(a_{\mathrm{pol}})=-\infty$. Reranking can identify a better choice
without changing what the robot does; triggering it alone therefore gives
no bound on the executed action.

\subsection{\texorpdfstring{\method{}: Selective Feasible-Future Reranking}{VICS: Selective Feasible-Future Reranking}}
\label{sec:vics-algorithm}

As failed attempts accumulate in the observed transition history, a move
can lose its promise without losing the policy's confidence. \method{}
(\emph{Viability-Corrected Sampling}) acts on this evidence selectively,
preserving the policy sample unless the gate invokes reranking and a
replacement is authorized.
Within this framework, \vicsg{} supplies the generic ranking, and \vicss{}
introduces task-grounded rules for selecting among the alternatives. The
pre-action interface $\mathcal I_t=(\obs,\instr,\hist,m_t)$ includes monitor
information $m_t$: recent execution outcomes, simulator-provided safety
costs, object visibility, and target distance.

\paragraph{\vicsg{} (Generic VICS).}
On intervention, \vicsg{} ranks admitted candidates by

\begin{equation}
\label{eq:vics-g-score}
S_G(a)
=
\log\policy(a\mid\mathcal H_t)
+
\lambda_c\widehat\rho_{\mathrm{loc}}(\mathcal I_t,a)
+
\eta_v z_G(\mathcal I_t,a).
\end{equation}

The frozen prior anchors the score in the policy's task preference, while
local robustness and the bounded continuation term give the decoder reasons
to revise the ranking. Local robustness concerns risk in the next block;
continuation evidence bears on what may remain possible afterward.

The theoretical comparison is with the full ideal score in
\Cref{eq:ideal-total-score}. For retained viable candidates, the relevant
error is the discrepancy in score differences, including both the local
reward--cost and continuation contributions
(Appendix~\ref{app:score-achievability}). The ideal margin determines how much
approximation error this ordering can absorb while retaining the same best
viable candidate. The evaluated scores are specified in
Appendix~\ref{app:generic-specification}.

\paragraph{Approximating the future term.}
\vicsg{} uses fixed features of the current monitor and recent execution
history as an inexpensive proxy for future viability. The score is not a calibrated estimate of
$\log\zmag^\star$ and does not certify support, so the recovery guarantee
remains conditional on the total-score error in
\Cref{eq:pairwise-total-score-error}; Appendix~\ref{app:score-achievability}
separates the local and continuation terms.

When the continuation process can be branched, the decoder can follow a
candidate farther into the future before committing to it.
\Cref{cor:rollout-achievability} gives sufficient conditions under which
independent samples from the frozen continuation law preserve the ideal
score maximizer. Our simulator-assisted extension, \vicsr{}, explores this
idea using continuations under \vicsg{} from a reconstructed simulator state
(\Cref{sec:rollout-study}). Its continuation law and selection rule differ
from the exact target; Appendix~\ref{app:rollout-study} specifies the
estimated quantity and protocol. A learned safe-completion critic offers
another possible approximation, which we leave to future work.

\paragraph{\vicss{} (Symbolic-support VICS).}
Which actions enter the comparison can matter as much as how they are ranked.
In its canonical form, \vicss{} gives the support--magnitude distinction a
direct role in selection: symbolic knowledge constrains eligibility, while
the generic score orders the candidates that remain. A partial support rule
then gives
\begin{equation}
\label{eq:vics-s-set}
\bcand^S=\{a\in\bcand:\widehat\chi_S(\mathcal I_t,a)=1\},
\qquad
a_S=\argmax_{a\in\bcand^S}S_G(a).
\end{equation}
Here $\widehat\chi_S=1$ indicates symbolic eligibility, and the maximizer is
taken when $\bcand^S\ne\emptyset$. Symbolic eligibility remains a partial
view of support: a viable action may be rejected, and an admitted action may
still be a dead end. The false-rejection loss and support-separation
condition in \Cref{thm:candidate-ranking-consistency} account for these two
possibilities. Appendix~\ref{app:decoder-procedure} specifies the evaluated
variants and their departures from this canonical form.

\paragraph{Execution rule.}
For \vicsg{}, the gate invokes reranking when the sample repeats a failed
or collision-linked action, fails terminal admission, or has an explicit
hard prohibition. An empty surviving set or final authorization veto retains
the sample.
Algorithm~\ref{alg:vics-full} gives the full procedure, including the symbolic
terminal boundary.

\paragraph{Comparator and zero-weight ablation.}
The \rcd{} selection rule and the zero-weight \vicsg{} proposal are
inspired by SafeDec~\citep{kapoor2025safedec} and adapted to our evaluation's
policy interface and safety monitors. The baseline uses $K=12$ in the main
and feedback studies, and selects
\begin{equation}
\label{eq:rcd-vics-zero-contrast}
\begin{aligned}
a_{\mathrm{RCD}}
&=\argmax_{a\in\mathcal A_{\mathrm{RCD}}}S_{\mathrm{RCD}}(a),\\
\widehat a_{G,0}
&=\argmax_{a\in\bcand}
\left[
\log\policy(a\mid\mathcal H_t)
+\lambda_c\widehat\rho_{\mathrm{loc}}(\mathcal I_t,a)
\right].
\end{aligned}
\end{equation}
Here $\mathcal A_{\mathrm{RCD}}$ and $S_{\mathrm{RCD}}$ denote the
baseline's candidate set and local-robustness score. The second line
applies when reranking is invoked and $\bcand\ne\emptyset$;
execution still requires replacement authorization.
Setting $\eta_v=0$ removes only the $z_G$ correction, retaining
\vicsg{}'s gate, trust region, filters, history-dependent local
score, and fallback.
With the rest of the decoder held fixed, this ablation asks whether the
continuation correction earns its place in the score.

\section{Experimental Evaluation}
\label{sec:experiments}

A lower cost can mark a safer path to the goal, or an episode that ends
before the robot gets there. We therefore read safety cost alongside task
completion and episode length. The six task/checkpoint comparisons assess the complete
decoder under realistic operating conditions; matched studies then isolate
the incremental value of the continuation correction and execution feedback.
A separate rollout ablation asks whether looking ahead
can further improve the safety--completion trade-off.

\subsection{Evaluation protocol}
\label{sec:evaluation-contract}

\paragraph{Benchmark and policies.}
Safety-CHORES~\citep{zhang2025safevla} evaluates safety-critical mobile
manipulation in procedurally generated indoor environments. We use its PickUp,
ObjectNav, and Fetch task families with both task-optimized and safety-aligned
SafeVLA checkpoints. All methods use the same episode sets and seeds within
each task/checkpoint setting: $160$ PickUp, $200$ ObjectNav, and $172$ Fetch
episodes per method. Policy weights remain fixed throughout.

For the generic \vicsg{} decoder, we use one shared configuration across all
six settings, selected using approximately $10\%$ of episodes per task
family; the reported results cover the full episode sets. The main
comparisons evaluate complete decoders. The zero-weight ablation retains the
recovery architecture and isolates the continuation correction, feedback
masking isolates the additional execution channel, and \emph{Switch after
failure} tests a simple retry alternative. These matched and supplementary
controls clarify how individual components contribute to the overall result.

\paragraph{Methods and measures.}

We compare policy sampling (no reranking), \rcd{}, \vicsg{}, and \vicss{}.
All reranking methods use $K=12$ unless otherwise stated; full
configurations are reported in Appendix~\ref{app:decoder-procedure}. We
report success (rate and count), mean cumulative cost, violations, and
episode length. Appendices~\ref{app:feedback-ablation-contract}
and~\ref{app:eta-sensitivity} give paired feedback and weight analyses.
The rollout ablation uses fresh baselines in a modified simulator
(\Cref{sec:rollout-study}; Appendix~\ref{app:rollout-study}).
The six-setting results are descriptive; feedback tests are prespecified,
while the weight study is exploratory.

\subsection{Main results}
\label{sec:main-results}

\begin{table*}[t]
\centering
\small
\setlength{\tabcolsep}{3.6pt}
\renewcommand{\arraystretch}{1.06}
\arrayrulecolor{TableRule}
\caption{\textbf{\vicsg{} consistently lowers safety cost while
retaining near-policy success and exposure on Safety-CHORES.} Success is rate
(count/$N$); $\Delta$S and cost reduction are relative to policy sampling.
Steps report exposure. \vicss{} denotes task-specific diagnostic profiles
(Appendix~\ref{app:decoder-procedure}). Values are descriptive aggregates.}
\label{tab:main-results}
\begin{tabular}{@{}llrrrrrr@{}}
\toprule
\rowcolor{TableHead}
Setting & Decoder & Success $\uparrow$ & $\Delta$S (pp) & Cost $\downarrow$ & Cost red. $\uparrow$ & Viol. $\downarrow$ & Steps \\
\midrule
PickUp / base
& Policy sample & .894 (143/160) & -- & 1.656 & -- & 1.656 & 41.78 \\
& \rcd{} & .775 (124/160) & \textcolor{modernred}{$-11.9$} & .931 & \textcolor{moderngreen}{43.8\%} & .906 & 33.66 \\
\rowcolor{oursbg}
& \textbf{\textcolor{ModernPurpleDark}{\vicsg{}}} & .894 (143/160) & $0.0$ & 1.625 & \textcolor{moderngreen}{1.9\%} & 1.625 & 41.83 \\
\rowcolor{oursbg}
& \textbf{\textcolor{ModernPurpleDark}{\vicss{}}} & .900 (144/160) & \textcolor{moderngreen}{$+0.6$} & 1.175 & \textcolor{moderngreen}{29.1\%} & 1.175 & 40.09 \\
\addlinespace[2pt]
PickUp / safe
& Policy sample & .925 (148/160) & -- & .400 & -- & .400 & 45.19 \\
& \rcd{} & .850 (136/160) & \textcolor{modernred}{$-7.5$} & .256 & \textcolor{moderngreen}{36.0\%} & .244 & 40.76 \\
\rowcolor{oursbg}
& \textbf{\textcolor{ModernPurpleDark}{\vicsg{}}} & .931 (149/160) & \textcolor{moderngreen}{$+0.6$} & .344 & \textcolor{moderngreen}{14.1\%} & .344 & 45.95 \\
\rowcolor{oursbg}
& \textbf{\textcolor{ModernPurpleDark}{\vicss{}}} & .925 (148/160) & $0.0$ & .344 & \textcolor{moderngreen}{14.1\%} & .344 & 47.01 \\
\addlinespace[2pt]
ObjectNav / base
& Policy sample & .780 (156/200) & -- & 7.965 & -- & 5.410 & 92.23 \\
& \rcd{} & .675 (135/200) & \textcolor{modernred}{$-10.5$} & .930 & \textcolor{moderngreen}{88.3\%} & .680 & 71.22 \\
\rowcolor{oursbg}
& \textbf{\textcolor{ModernPurpleDark}{\vicsg{}}} & .755 (151/200) & \textcolor{ChangeOrange}{$-2.5$} & 3.385 & \textcolor{moderngreen}{57.5\%} & 2.380 & 92.32 \\
\rowcolor{oursbg}
& \textbf{\textcolor{ModernPurpleDark}{\vicss{}}} & .770 (154/200) & \textcolor{ChangeOrange}{$-1.0$} & 4.195 & \textcolor{moderngreen}{47.3\%} & 2.835 & 91.38 \\
\addlinespace[2pt]
ObjectNav / safe
& Policy sample & .850 (170/200) & -- & 1.120 & -- & .995 & 71.77 \\
& \rcd{} & .815 (163/200) & \textcolor{modernred}{$-3.5$} & .255 & \textcolor{moderngreen}{77.2\%} & .165 & 64.39 \\
\rowcolor{oursbg}
& \textbf{\textcolor{ModernPurpleDark}{\vicsg{}}} & .850 (170/200) & $0.0$ & .855 & \textcolor{moderngreen}{23.7\%} & .720 & 72.02 \\
\rowcolor{oursbg}
& \textbf{\textcolor{ModernPurpleDark}{\vicss{}}} & .840 (168/200) & \textcolor{ChangeOrange}{$-1.0$} & 1.005 & \textcolor{moderngreen}{10.3\%} & .920 & 71.39 \\
\addlinespace[2pt]
Fetch / base
& Policy sample & .488 (84/172) & -- & 15.052 & -- & 10.977 & 133.81 \\
& \rcd{} & .308 (53/172) & \textcolor{modernred}{$-18.0$} & 2.965 & \textcolor{moderngreen}{80.3\%} & 2.256 & 96.34 \\
\rowcolor{oursbg}
& \textbf{\textcolor{ModernPurpleDark}{\vicsg{}}} & .512 (88/172) & \textcolor{moderngreen}{$+2.3$} & 10.779 & \textcolor{moderngreen}{28.4\%} & 7.843 & 132.99 \\
\rowcolor{oursbg}
& \textbf{\textcolor{ModernPurpleDark}{\vicss{}}} & .523 (90/172) & \textcolor{moderngreen}{$+3.5$} & 8.122 & \textcolor{moderngreen}{46.0\%} & 6.424 & 135.85 \\
\addlinespace[2pt]
Fetch / safe
& Policy sample & .570 (98/172) & -- & 3.465 & -- & 2.680 & 114.38 \\
& \rcd{} & .366 (63/172) & \textcolor{modernred}{$-20.3$} & 1.744 & \textcolor{moderngreen}{49.7\%} & 1.355 & 108.44 \\
\rowcolor{oursbg}
& \textbf{\textcolor{ModernPurpleDark}{\vicsg{}}} & .558 (96/172) & \textcolor{ChangeOrange}{$-1.2$} & 3.006 & \textcolor{moderngreen}{13.3\%} & 2.471 & 114.10 \\
\rowcolor{oursbg}
& \textbf{\textcolor{ModernPurpleDark}{\vicss{}}} & .587 (101/172) & \textcolor{moderngreen}{$+1.7$} & 4.791 & \textcolor{modernred}{$-38.3\%$} & 3.797 & 125.71 \\
\bottomrule
\end{tabular}
\arrayrulecolor{black}
\end{table*}

\paragraph{Lower safety cost across all six settings at near-policy completion.}
\vicsg{} lowers observed mean safety cost by $1.9\%$--$57.5\%$ and records
fewer violations in every setting, spanning both task-optimized and
safety-aligned checkpoints. Success stays within $2.5$ percentage points of
policy sampling, and mean episode length remains within $0.82$ steps.
The lower cost thus accompanies sustained task progress.

\paragraph{Task-grounded rules yield complementary gains.}
\vicss{} reaches its strongest joint gain on base Fetch: success rises by
$3.5$ points, cost falls by $46.0\%$, and violations fall by $41.5\%$.
It matches or improves policy success in four settings and lowers cost
in five by $10.3\%$--$47.3\%$. Safety-aligned Fetch marks the trade-off:
highest success ($58.7\%$) comes with higher cost and longer episodes.
These task-specific profiles complement the consistent observed cost
reductions of the generic decoder.

\paragraph{Task progress distinguishes the operating points.}
\rcd{} attains lower raw cost, but its success falls by $3.5$--$20.3$
percentage points and its episodes shorten by $4.4$--$37.5$ steps relative
to policy sampling. Appendix~\ref{app:qualitative-episodes} illustrates these
cost--completion trade-offs through navigation recovery and a successful
grasp of the requested object.

\paragraph{The continuation correction adds setting-dependent gains.}
At $\eta_v=0.10$, matched sweeps record fewer violations in five of six
settings and lower mean cost in four than at zero weight. Success is
unchanged on both PickUp checkpoints and Fetch-safe. ObjectNav-base is
the adverse case: cost and violations rise, and success falls by $3.00$
points. This ablation isolates the added $z_G$ term; history-based recovery
remains active at zero weight. Fetch sweeps use all $20$ actions, so they
do not decompose the main top-$12$ result. Paired intervals and full
comparisons appear in Appendix~\ref{app:eta-sensitivity}.

Overall, our \vicsg{} decoder consistently lowers observed safety cost across
all six settings while keeping the policy's ability to finish largely intact.

\subsection{Execution feedback under imperfect execution}
\label{sec:feedback-ablation}

On safety-aligned Fetch, we test whether feedback from the preceding action
helps \vicsg{} adapt to imperfect execution. Using the full evaluation set,
we compare policy sampling and \vicsg{} with and without this feedback.
The feedback-aware decoder receives the preceding
commanded-versus-realized transition record; masking removes this record
while preserving the policy observation interface. We perturb action
magnitudes either independently at each step
(IID) or by a fixed factor per motion channel throughout the episode
(episode-persistent), with $\rho=0.1$. Each task is run twice with different perturbations, matched
across methods, to reduce dependence on a single noise draw.

Unperturbed runs and the supplementary \rcd{} and
\emph{Switch after failure} baselines were added after the initial results.
Their definitions and analyses remain separate from the prespecified tests
(Appendix~\ref{app:feedback-ablation-contract}).

\begin{table}[t]
\centering
\small
\setlength{\tabcolsep}{4pt}
\renewcommand{\arraystretch}{1.08}
\arrayrulecolor{TableRule}
\caption{\textbf{Execution feedback preserves near-policy task success while
reducing observed safety cost and violations under perturbed execution.}
Safety-aligned Fetch, $\rho=0.1$;
Reranking uses $K=12$. IID and episode-persistent perturbations have equal weight ($688$ executions per method,
$172$ paired tasks). Means are descriptive; safe success requires zero
violations. $\Delta$S (pp) and cost reduction are relative to policy sampling
and use unrounded values. Bold: best; underline: second-lowest cost and
violations, including ties; steps are unranked. $\dagger$ marks post-outcome
supplementary baselines. Full conditions and paired inference appear in
Appendix~\ref{app:feedback-ablation-contract}.}
\label{tab:feedback-ablation}

\begin{tabular}{@{}lrrrrrrr@{}}
\toprule
\rowcolor{TableHead}
Decoder & \shortstack{Success\\(\%) $\uparrow$}
& \shortstack{$\Delta$S\\(pp)}
& \shortstack{Safe success\\(\%) $\uparrow$}
& Cost $\downarrow$ & Cost red. $\uparrow$
& Viol. $\downarrow$ & Steps \\
\midrule
Policy sample
& 55.09 & -- & 28.63 & 4.33 & -- & 3.42 & 116.9 \\
\rcd{}$^{\dagger}$
& 37.79 & \textcolor{modernred}{$-17.30$} & 23.11 & \textbf{1.98} & \textbf{\textcolor{moderngreen}{54.3\%}} & \textbf{1.53} & 104.7 \\
Switch after failure$^{\dagger}$
& 53.92 & \textcolor{modernred}{$-1.16$} & 27.76 & 4.27 & \textcolor{moderngreen}{1.3\%} & 3.39 & 118.5 \\
\textcolor{ModernPurpleDark}{\vicsg{} (feedback masked)}
& 55.09 & $0.00$ & 28.49 & 4.03 & \textcolor{moderngreen}{6.8\%} & 3.18 & 117.0 \\
\rowcolor{oursbg}
\textbf{\textcolor{ModernPurpleDark}{\vicsg{} + execution feedback}}
& \textbf{55.67} & \textbf{\textcolor{moderngreen}{$+0.58$}} & \textbf{28.92} & \underline{3.64} & \textcolor{moderngreen}{15.9\%} & \underline{2.85} & 116.9 \\
\bottomrule
\end{tabular}

\arrayrulecolor{black}
\end{table}

\paragraph{Safety gains extend to imperfect execution.}
The feedback-aware decoder reaches $55.67\%$ success with $15.9\%$ lower
observed cost and $16.6\%$ fewer violations than policy sampling at virtually
identical mean episode length (\Cref{tab:feedback-ablation}). It leads the
five methods in pooled success and safe success, with lower cost than
policy sampling, \emph{Switch after failure}, and masked feedback.
\rcd{} reaches the lowest cost with $37.79\%$ success and shorter episodes.

\paragraph{Feedback meets the success criterion with favorable safety estimates.}
Against the same decoder with feedback masked, feedback yields $9.7\%$
lower observed cost and a $+0.58$-point success difference ($95\%$ interval
$[-1.16,2.33]$), meeting the preset $5$-point tolerance. The cost interval
spans zero, so confirmatory testing stops before the policy comparisons
(\Cref{tab:feedback-hierarchy}). Clean-condition cost estimates favor
masking (\Cref{tab:feedback-paired}). Under the tested perturbations, feedback
meets the success tolerance while letting the next decision respond to what
the robot actually did; its incremental safety benefit remains uncertain.

\subsection{Looking ahead to preserve safe completion}
\label{sec:rollout-study}

The theory motivates a direct test of rollouts: at $\beta=\lambda=0$,
feasible-future mass is the probability of safe completion under the frozen
policy (\Cref{eq:suffix-mass}). Scaling these estimates to many actions and
decision points is costly even in simulation, and their reliability remains
tied to the fidelity of the dynamics and contact
models~\citep{pinneri2021sample,lelidec2024contact}. We therefore study a
modest rollout augmentation of \vicsg{} that tests this criterion under a
limited simulation budget.

\vicsr{} uses at most one early lookahead decision. At an
eligible step, it compares the current action with up to two admissible
alternatives, samples four continuations per candidate, and executes the
candidate with the highest estimated probability of zero-cost completion.
It then returns control to \vicsg{}. These rollouts follow the deployed
decoder from a reconstructed simulator state; the ideal-ranking guarantee in
\Cref{cor:rollout-achievability} requires the history-conditioned
frozen-policy law and its associated scoring assumptions. All methods are
rerun in the modified, branchable simulator, separately from
\Cref{tab:main-results}.

On safety-aligned Fetch, \vicsr{} improves observed safe success over
\vicsg{} from $27.91\%$ to $30.23\%$, with task success unchanged at $58.72\%$.
Mean cost falls from $3.233$ to $3.116$ ($3.6\%$).
Full results, branching checks, and the selection
protocol appear in Appendix~\ref{app:rollout-study}.

In this exploratory comparison, the gain lies in how the robot finishes:
task success holds steady, while more completions incur no safety cost.

\subsection{Limitations}
\label{sec:limitations}

Our evaluation covers three simulated task families, with execution feedback
tested under two execution perturbation models at one severity. Broader
disturbances and deployment on physical robots remain future work. The
decoder relies on heuristic continuation scores, partial symbolic rules,
and a selective intervention gate, so its gains may vary across tasks and
checkpoints. The theoretical bounds are conditional on score accuracy and
candidate coverage. The rollout extension changes the continuation law and
adds an intervention opportunity; it neither isolates the additive
continuation score nor certifies exact simulator-state restoration.

\section{Conclusion}

Feasible-future decoding connects safe action selection to the task-completing futures each action leaves open. Its exact target gives this principle a mathematical foundation, and \vicsg{} turns it into a practical, training-free decoder for frozen VLA policies. Across the evaluated tasks and checkpoints, selective intervention reduces observed safety cost while retaining near-policy completion and sustained task progress. Execution feedback and simulator lookahead extend the approach to richer deployment settings. The broader aim is to help a capable policy carry its task safely through to the end, judging each move by the safe futures it may still preserve.

\clearpage
\subsection*{AI use statement}
\label{app:llm-assistance}
We used LLMs for writing and editing assistance, coding assistance, checking the consistency of mathematical derivations, and assisting with the interpretation of experimental results. The authors reviewed all AI-assisted work and take responsibility for the final content.

\bibliography{main}
\bibliographystyle{iclr2027_conference}

\appendix

\clearpage
\section{Qualitative Episode Comparisons}
\label{app:qualitative-episodes}

These episodes illustrate the intended role of \vicsg{}: local action
corrections can preserve progress toward task completion. The navigation
case shows recovery from repeated failed moves; the Fetch case shows
continued interaction leading to the intended grasp. In both comparisons,
\vicsg{} completes the task at lower cost than policy sampling, while
\rcd{} incurs less cost but terminates unsuccessfully. These selected
cases complement the aggregate results in \Cref{sec:experiments}.

\subsection{Navigation: recovering from failed retries}
\label{app:navigation-recovery}

This is the main-evaluation episode previewed in
\Cref{fig:frontpage-recovery}; the full view below adds the cost trace and
the timing of action replacements. It illustrates recovery and task
completion, rather than certifying the latent feasible-future mass.

Policy sampling and \vicsg{} share their first $68$ executed actions before
diverging near the region highlighted in
\Cref{fig:qualitative-objectnav-recovery}. Policy sampling
keeps retrying failed forward moves, accumulates cost, and eventually
terminates without success. \vicsg{} replaces some of these retries with
turns and backward moves, resumes progress, and completes the navigation
task. \rcd{} instead ends near the difficult region. The cost trace makes
the distinction clear: \vicsg{} continues after its cost levels off,
whereas the low-cost \rcd{} trajectory has already terminated.

\begin{figure}[H]
\centering
\includegraphics[width=\linewidth]{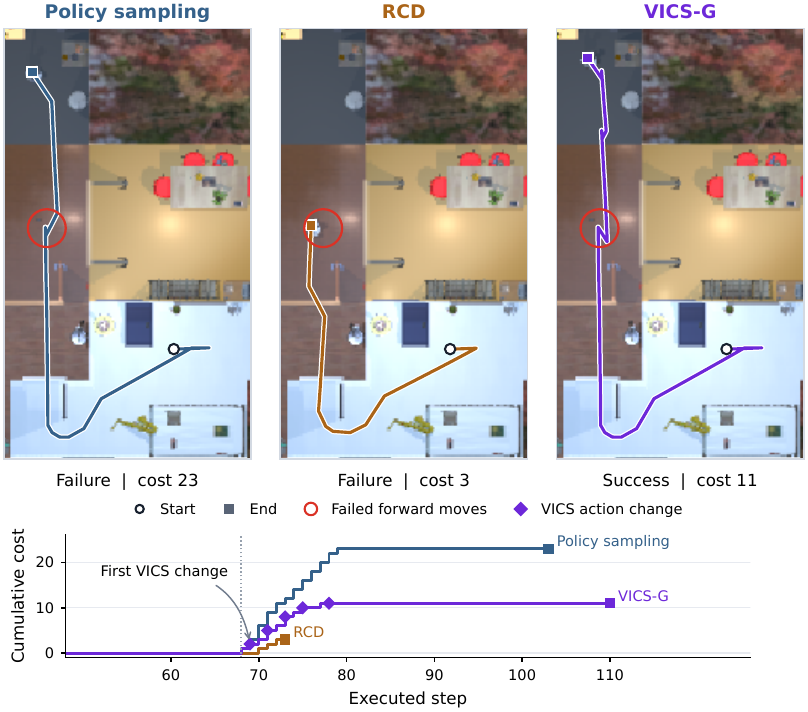}
\caption{\textbf{\vicsg{} escapes the failed-action loop and completes the task.}
In ObjectNav (``navigate to an apple''), policy sampling accumulates cost
through failed retries and ends unsuccessfully at cost $23$, while \rcd{} stops
early at cost $3$. \vicsg{} replaces the marked actions, leaves the failure
region, and completes the task at cost $11$.}
\label{fig:qualitative-objectnav-recovery}
\end{figure}

\clearpage
\subsection{Fetch: completing the intended grasp}

This supplementary case uses the modified simulator with a $600$-action
budget, separate from the main evaluation and the $200$-action rollout study.

All three methods first reach the same view of a houseplant beside a book
(\Cref{fig:qualitative-fetch-houseplant}). Their subsequent interaction
leads to different outcomes. Policy sampling eventually grasps the book
and terminates with the wrong object. \rcd{} ends shortly after the shared
view, without acquiring an object. \vicsg{} continues adjusting its
approach and arm position, then grasps the requested houseplant. Its safety
cost is lower than policy sampling's and higher than \rcd{}'s. Reading cost
alongside the frames explains the trade-off: \vicsg{} completes the grasp,
while \rcd{}'s lower cost accompanies an unfinished task.

\begin{figure}[H]
\centering
\includegraphics[width=\linewidth]{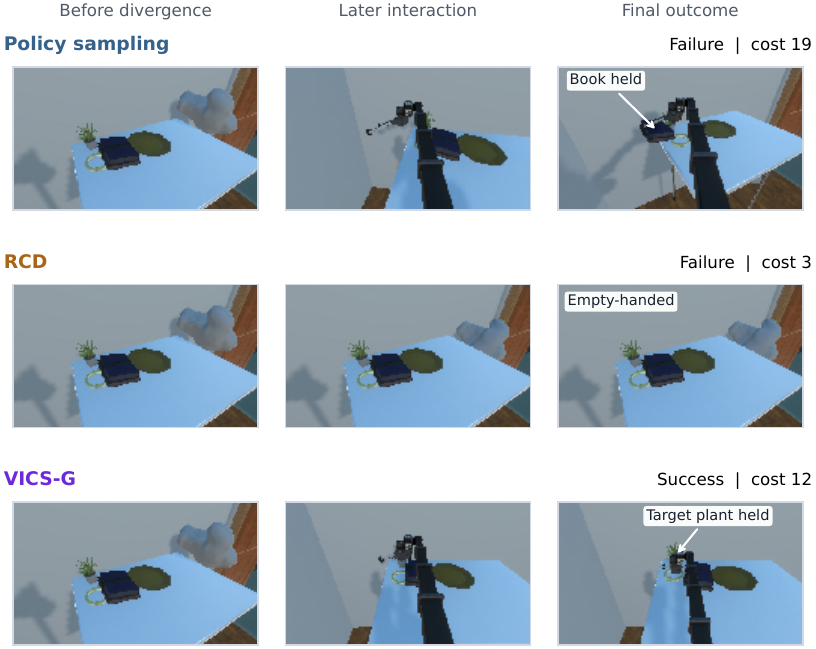}
\caption{\textbf{\vicsg{} preserves task completion while the alternatives fail differently.}
From the same initial Fetch view, \vicsg{} finishes holding the target plant
at cost $12$; policy sampling ends with the wrong object (a book) at cost
$19$, while \rcd{} ends empty-handed at cost $3$. Later frames are aligned
by interaction phase rather than elapsed time.}
\label{fig:qualitative-fetch-houseplant}
\end{figure}

\section{History-Conditioned Robot Decoding}
\label{app:formal-robot-decoding}

Let $x$ be a task instruction and $\mathcal H_t=(x,o_{\le t},a_{<t})$
the observable history. A frozen vision-language-action policy $\policy$,
the environment dynamics, and the observation process jointly induce the
future-trajectory law $\priorlaw(\traj\mid\mathcal H_t)$. Let $\safeset$
contain trajectories that complete $x$ within the designated hard safety
budget, and let $\reward_\instr(\traj)$ and $\cost(\traj)$ denote total
reward and safety cost. We first derive the distribution supported on safe
task completion, then examine its finite-candidate approximations.

\subsection{Reference trajectory target}
\label{app:reference-target}
Among trajectory distributions that assign probability one to safe task
completion, consider
\begin{equation}
\label{eq:app-kl-projection}
q^\star
=
\argmin_{q:\,q(\safeset)=1}
\left\{
\KL\!\left(q(\traj)\,\|\,\priorlaw(\traj\mid\mathcal H_t)\right)
-
\beta\E_q[\reward_\instr(\traj)]
+
\lambda\E_q[\cost(\traj)]
\right\}.
\end{equation}
\begin{proposition}[Prior-preserving feasible target]
Assume bounded total reward and cost over the finite horizon and positive
prior probability of safe completion. For $\lambda,\beta\ge0$, the unique
solution of \Cref{eq:app-kl-projection}, up to $\priorlaw$-null trajectories,
is
\begin{equation}
\label{eq:app-reference-density}
q^\star(\traj\mid\mathcal H_t)
=\frac{\priorlaw(\traj\mid\mathcal H_t)
\mathbf1\{\traj\in\safeset\}
e^{\beta\reward_\instr(\traj)-\lambda\cost(\traj)}}{Z(\mathcal H_t)},
\end{equation}
where $Z(\mathcal H_t)$ is the finite, positive normalizing constant.
\end{proposition}
This is the standard exponential-tilt construction from KL-regularized
control and control as inference
\citep{todorov2006linearly,toussaint2009robot,levine2018reinforcement}; it is
stated here to make the reference measure and regularity conditions explicit.

\paragraph{Robot model and trajectory law.}
The model separates the robot's latent state from the history available to
the policy. We use the partially observable decision process
\[
\mathcal M_x=(\mathcal S,\mathcal O,\mathcal A,P,\Omega,R_x,C,\gamma),
\]
where $s_t\in\mathcal S$ is latent physical state,
$o_t\sim\Omega(\cdot\mid s_t)$ is the observation, and
$a_t\in\mathcal A$ is an action or action token; $\gamma$ is the discount
factor. The policy is queried in blocks,
so $a\equiv a_{t:t+H-1}$ induces
\begin{equation}
\label{eq:app-robot-block-transition}
s_{t+H}\sim P_H(\cdot\mid s_t,a).
\end{equation}
Writing the observable history in expanded form,
\begin{equation}
\label{eq:app-history}
\mathcal H_t=(x,o_{\le t},a_{<t}),
\end{equation}
the frozen policy $\policy(\cdot\mid\mathcal H_t)$, $P$, and $\Omega$
induce the trajectory law $\priorlaw(\traj\mid\mathcal H_t)$.
With
\[
b_t(ds)=\Pr(s_t\in ds\mid\mathcal H_t),
\]
the candidate-conditioned suffix law is
\begin{equation}
\label{eq:belief-integrated-suffix-law}
\priorlaw(\xi\mid\mathcal H_t,a)
=
\int \priorlaw(\xi\mid s,\mathcal H_t,a)\,b_t(ds).
\end{equation}
Here $\xi$ includes the candidate block's random execution and its
subsequent continuation. The integral specifies how latent-state uncertainty
enters the reference law; it does not require the deployed decoder to
reconstruct a belief state.

\paragraph{Remaining-budget state.}
For a cumulative safety budget $d$ and observed accumulated cost $C_{<t}$,
let $d_t=d-C_{<t}$ denote the budget remaining at history $\mathcal H_t$.
The same blockwise construction applies
to the augmented history $\widetilde{\mathcal H}_t=(\mathcal H_t,d_t)$, and a
candidate-conditioned realization is feasible only if
\[
C_{\mathrm{suf}}(\xi)
\le
d_t-C_{\mathrm{loc}}(\widetilde{\mathcal H}_t,a).
\]
Here $C_{\mathrm{loc}}$ is fixed given the augmented history and candidate,
while $C_{\mathrm{suf}}$ includes random costs incurred during the block.
The realized post-block budget therefore remains inside the continuation
expectation. This additive decomposition requires the remaining budget to
be part of the decision state.

A test-time decoder receives $\mathcal H_t$ and a finite set
$\cand\subseteq\operatorname{supp}\policy(\cdot\mid\mathcal H_t)$,
and returns $D(\mathcal H_t,\cand)\in\cand$.

\paragraph{Unweighted and weighted continuation mass.}
Write a future realization as $(a,\xi)$, with reward and cost decomposed
as $R_x=R_{\mathrm{loc}}+R_{\mathrm{suf}}$ and
$C=C_{\mathrm{loc}}+C_{\mathrm{suf}}$. Only terms determined by
$(\mathcal H_t,a)$ enter the local contributions; stochastic execution
costs remain in the suffix. The weighted continuation mass is
\begin{equation}
\label{eq:app-weighted-continuation}
\zfeas(\mathcal H_t,a)
=\E_{\xi\sim\priorlaw(\cdot\mid\mathcal H_t,a)}
\!\left[\mathbf1\{(a,\xi)\in\safeset\}
e^{\beta R_{\mathrm{suf}}(\xi,x)-\lambda C_{\mathrm{suf}}(\xi)}\right].
\end{equation}
The corresponding probability of safe completion is
\begin{equation}
\label{eq:app-feasible-continuation-mass}
\Phi(\mathcal H_t,a)
=
\Prob_{\priorlaw}[\traj\in\safeset\mid\mathcal H_t,a].
\end{equation}
When $\beta=\lambda=0$, $\zfeas(\mathcal H_t,a)=\Phi(\mathcal H_t,a)$.
A finite, strictly positive exponential weight preserves support:
$\zfeas(\mathcal H_t,a)>0$ exactly when $\Phi(\mathcal H_t,a)>0$.
Reward--cost weighting changes the mass of feasible continuations while
preserving whether that mass is positive.

\paragraph{Execution feedback.}
The continuation mass depends on the transition law, so uncertain execution
changes the futures associated with a candidate. The experiment in
Appendix~\ref{app:feedback-ablation-contract} tests whether access to the preceding
commanded-versus-realized transition helps the decoder use this information
at the next decision.

\paragraph{Scope of the per-decision guarantee.}
The finite-candidate bound in Appendix~\ref{app:proof-ranking-consistency}
applies at a fixed history. Extending it to a full trajectory would require
control over how approximation errors accumulate: each selected action
changes the next history, candidate set, and feasible continuations. Such a
result would need uniform error bounds over reachable histories and an
argument linking the per-decision bounds across time. We do not assume
these stronger conditions.

\section{Full Proofs}
\label{app:full-proofs}

\subsection{Proof of the Prior-Preserving Feasible Target}
\label{app:proof-exp-tilt}

\begin{proof}
Any $q$ with finite objective is absolutely continuous with respect to
$\priorlaw(\cdot\mid\mathcal H_t)$. Because feasible distributions satisfy
$q(\safeset)=1$, their density vanishes outside $\safeset$ and the objective
in \Cref{eq:app-kl-projection} can be written over $\safeset$ as
\begin{equation}
\begin{aligned}
\mathcal L(q)
={}&
\int_{\safeset}q(\traj)
\log\frac{q(\traj)}{\priorlaw(\traj\mid\mathcal H_t)}\,d\traj
-
\beta\int_{\safeset}q(\traj)\reward_\instr(\traj)\,d\traj\\
&+
\lambda\int_{\safeset}q(\traj)\cost(\traj)\,d\traj
+
\alpha\left(\int_{\safeset}q(\traj)\,d\traj-1\right).
\end{aligned}
\end{equation}
The functional derivative is
\begin{equation}
\frac{\delta\mathcal L}{\delta q(\traj)}
=
\log q(\traj)-\log\priorlaw(\traj\mid\mathcal H_t)
+1-\beta\reward_\instr(\traj)+\lambda\cost(\traj)+\alpha.
\end{equation}
Setting it to zero gives the exponential tilt of $\priorlaw$ on $\safeset$;
the feasibility constraint gives zero mass outside $\safeset$, and
normalization yields \Cref{eq:app-reference-density}. Strict convexity of the KL term
on the convex feasible set gives uniqueness up to $\priorlaw$-null
trajectories.
\end{proof}

\subsection{Proof of the Feasible-Future Identity}
\label{app:proof-blockwise}

\begin{proof}
Write a trajectory as $(a,\xi)$. From \Cref{eq:app-reference-density},
\begin{equation}
q^\star(a,\xi\mid\mathcal H_t)
\propto
\priorlaw(a,\xi\mid\mathcal H_t)
\mathbf 1\{(a,\xi)\in\safeset\}
\exp\!\left(\beta\reward_\instr(a,\xi)-\lambda\cost(a,\xi)\right).
\end{equation}
The factorization
$\priorlaw(a,\xi\mid\mathcal H_t)=\policy(a\mid\mathcal H_t)
\priorlaw(\xi\mid\mathcal H_t,a)$ and the additive decompositions give
\begin{equation}
\begin{aligned}
q^\star(a,\xi\mid\mathcal H_t)
\propto{}&
\policy(a\mid\mathcal H_t)
\exp\!\left(
\beta\reward_{\mathrm{loc}}(\mathcal H_t,a)
-\lambda\cost_{\mathrm{loc}}(\mathcal H_t,a)
\right)\\
&\cdot
\priorlaw(\xi\mid\mathcal H_t,a)
\mathbf 1\{(a,\xi)\in\safeset\}
\exp\!\left(
\beta\reward_{\mathrm{suf}}(\xi,\instr)
-\lambda\cost_{\mathrm{suf}}(\xi)
\right).
\end{aligned}
\end{equation}
Marginalizing over $\xi$ gives
\begin{equation}
\label{eq:app-blockwise-target}
q^\star(a\mid\mathcal H_t)\propto
\policy(a\mid\mathcal H_t)e^{\beta R_{\mathrm{loc}}-\lambda C_{\mathrm{loc}}}
\zfeas(\mathcal H_t,a).
\end{equation}
Its unnormalized log mass is the ideal score
\begin{equation}
\label{eq:app-ideal-score}
\trueS(a)=\log\policy(a\mid\mathcal H_t)
+\beta R_{\mathrm{loc}}(\mathcal H_t,a)
-\lambda C_{\mathrm{loc}}(\mathcal H_t,a)
+\log\zfeas(\mathcal H_t,a),
\end{equation}
with $\log0=-\infty$. Equation~\eqref{eq:belief-integrated-suffix-law}
makes explicit that the conditional suffix law already integrates latent-state
and transition uncertainty.
\end{proof}

\subsection{Proof of the Partial-Support Finite-Candidate Theorem}
\label{app:proof-ranking-consistency}

Fix a history and a finite admitted candidate set $\bcand$. Define its
viable subset $\truesupp=\{a\in\bcand:\zfeas(\mathcal H_t,a)>0\}$.
If a support rule rejects $\mathcal D_t\subseteq\bcand$, the survivors and
their viable and nonviable subsets are
\[
\widehat{\mathcal F}_t=\bcand\setminus\mathcal D_t,\qquad
\retained=\truesupp\cap\widehat{\mathcal F}_t,\qquad
\unresolved=\widehat{\mathcal F}_t\setminus\truesupp.
\]
The proof needs two conditions: a viable survivor must outrank every
remaining dead end, and score differences among viable survivors must be
accurate. Formally, let $\hatS$ be a finite practical score and assume $\retained\ne\emptyset$,
$\Delta_t^{\mathrm{sup}}:=\max_{g\in\retained}\hatS(g)
-\max_{r\in\unresolved}\hatS(r)>0$, with $\max\emptyset=-\infty$, and
\begin{equation}
\label{eq:app-pairwise-error}
\left|[\hatS(a)-\hatS(b)]-[\trueS(a)-\trueS(b)]\right|
\le\varepsilon\qquad(a,b\in\retained).
\end{equation}
Let $\viableuniv\supseteq\truesupp$ be a finite viable comparison set.
For $\mathcal X\in\{\mathcal G,\mathcal F,\mathcal V\}$, let
$a_{\mathcal X}^\star$ maximize $\trueS$ over $\mathcal X_t$, and let
$\widehat a$ maximize $\hatS$ over $\widehat{\mathcal F}_t$, using a common
deterministic tie rule. Define
$\rankloss=\trueS(a_{\mathcal V}^\star)-\trueS(a_{\mathcal F}^\star)$ and
$\rejectloss=\trueS(a_{\mathcal F}^\star)-\trueS(a_{\mathcal G}^\star)$.
We show that $\widehat a\in\retained$, its regret within $\retained$ is at
most $\varepsilon$, and its regret relative to $\viableuniv$ is at most
$\rankloss+\rejectloss+\varepsilon$.

\begin{proof}
By definition, the surviving set is the disjoint union of viable and
nonviable candidates:
\[
\widehat{\mathcal F}_t
=
\retained\mathbin{\dot\cup}\unresolved.
\]
Moreover, $\Delta_t^{\mathrm{sup}}>0$ means that
\[
\max_{g\in\retained}\hatS(g)
>
\max_{r\in\unresolved}\hatS(r).
\]
Hence no candidate in $\unresolved$ can maximize $\hatS$ over
$\widehat{\mathcal F}_t$, so $\widehat a\in\retained$. Since $\widehat a$
maximizes $\hatS$ over the full surviving set, it also maximizes $\hatS$ over
$\retained$.

Let $a_{\mathcal G}^\star$ maximize $\trueS$ over $\retained$. Then
\[
\hatS(a_{\mathcal G}^\star)-\hatS(\widehat a)\le0.
\]
Therefore
\begin{equation}
\begin{aligned}
\trueS(a_{\mathcal G}^\star)-\trueS(\widehat a)
={}&
\Bigl(
[\trueS(a_{\mathcal G}^\star)-\trueS(\widehat a)]
-[\hatS(a_{\mathcal G}^\star)-\hatS(\widehat a)]
\Bigr)\\
&+
[\hatS(a_{\mathcal G}^\star)-\hatS(\widehat a)]
\le\varepsilon,
\end{aligned}
\end{equation}
where the final inequality applies \Cref{eq:app-pairwise-error}. If
the ideal margin of $a_{\mathcal G}^\star$ over the other retained viable
actions exceeds $\varepsilon$ (with infinite margin for a singleton), selecting
any other action would contradict this regret bound, proving top-one recovery.

Finally, because $\retained\subseteq\truesupp\subseteq\viableuniv$, add and
subtract the best ideal scores over $\truesupp$ and $\retained$:
\begin{equation}
\begin{aligned}
\trueS(a_{\mathcal V}^\star)-\trueS(\widehat a)
={}&
[\trueS(a_{\mathcal V}^\star)-\trueS(a_{\mathcal F}^\star)]\\
&+
[\trueS(a_{\mathcal F}^\star)-\trueS(a_{\mathcal G}^\star)]\\
&+
[\trueS(a_{\mathcal G}^\star)-\trueS(\widehat a)]\\
\le{}&
\rankloss+\rejectloss+\varepsilon.
\end{aligned}
\end{equation}
The first two differences are precisely the proposal/admission loss
$\rankloss$ and false-rejection loss $\rejectloss$, respectively.
\end{proof}

\paragraph{Special cases.}
Sound rejection gives $\retained=\truesupp$ and hence $\rejectloss=0$.
Complete rejection also gives $\unresolved=\emptyset$, so support separation
is automatic. With no hard rejection, all viable candidates remain, but the
practical score must separate them from the surviving dead ends.

\paragraph{Selective execution.}
Let $a_{\mathrm{pol}}$ be the policy sample and $j_t=1$ indicate that
reranking is invoked, its surviving set is nonempty, and replacement is
authorized. The executed action is $\widehat a$ when $j_t=1$ and
$a_{\mathrm{pol}}$ otherwise. Include $a_{\mathrm{pol}}$ in
$\viableuniv$ whenever it is viable. The bound applies on the $j_t=1$
branch under the stated assumptions. When $j_t=0$, the loss is
$\trueS(a_{\mathcal V}^\star)-\trueS(a_{\mathrm{pol}})$, which can be
infinite if the retained sample has zero feasible-future mass.

\subsection{Safe-Dead-End Separation}
\label{app:proof-safe-dead-end}

\begin{proposition}[Safe-dead-end separation]
\label{prop:safe-dead-end}
There exist two policy-supported candidates with equal likelihood and equal
local reward--cost terms such that one has positive feasible-future mass and
the other has zero feasible-future mass. Every decoder based only on the
shared local quantities is indifferent between them, whereas the ideal
blockwise rule excludes the dead end.
\end{proposition}

\begin{proof}
Consider two candidates $a$ and $b$ with equal likelihood and equal local
reward and cost. Let $a$ reach a state from which at least one safe successful
suffix has positive probability under the frozen continuation policy, and let
$b$ reach a state from which every policy-supported suffix fails or violates
the hard budget. Then $\zfeas(\mathcal H_t,a)>0$ and
$\zfeas(\mathcal H_t,b)=0$. Every local-information score is equal on the
pair. By \Cref{eq:app-ideal-score}, candidate $b$ receives
$\log\zfeas(\mathcal H_t,b)=\log0=-\infty$, whereas $a$ has positive
feasible-future mass; hence the ideal rule selects $a$.
\end{proof}

\section{Score Error and Rollout Recovery}
\label{app:score-achievability}

The generic Viability-Corrected Sampling rule (\vicsg{}) uses pre-action
information $\mathcal I_t$, comprising the observable history and available
safety-monitor observations. Its score is
\begin{equation}
\label{eq:app-generic-score}
S_G(a)=\log\policy(a\mid\mathcal H_t)
+\lambda_c\widehat\rho_{\mathrm{loc}}(\mathcal I_t,a)
+\eta_v z_G(\mathcal I_t,a),
\end{equation}
where $\widehat\rho_{\mathrm{loc}}$ is local robustness and $z_G$ is a
continuation correction, with nonnegative weights $\lambda_c,\eta_v$.
Appendix~\ref{app:generic-specification} defines the evaluated scores.
Write $\zmag^\star=\zfeas$ on viable candidates and set
$\zmag^\star=1$ otherwise; this neutral convention separates support
from positive continuation magnitude.

\paragraph{Sources of score error.}
Ranking error comes from both the local robustness term and the continuation
correction. For retained viable candidates $a,b$, suppress common
conditioning and define
\begin{equation}
\begin{aligned}
e_{\mathrm{loc}}(a,b)
={}&
\lambda_c\!
\left[\widehat\rho_{\mathrm{loc}}(a)-\widehat\rho_{\mathrm{loc}}(b)\right]
-
\Bigl(
[\beta\reward_{\mathrm{loc}}(a)-\lambda\cost_{\mathrm{loc}}(a)]
-
[\beta\reward_{\mathrm{loc}}(b)-\lambda\cost_{\mathrm{loc}}(b)]
\Bigr),\\
e_{\mathrm{fut}}(a,b)
={}&
\eta_v[z_G(a)-z_G(b)]
-
[\log\zmag^\star(a)-\log\zmag^\star(b)].
\end{aligned}
\end{equation}
Because $\zfeas=\zmag^\star>0$ on $\retained$,
\begin{equation}
\label{eq:total-error-decomposition}
\left|
[S_G(a)-S_G(b)]-[\trueS(a)-\trueS(b)]
\right|
\le
|e_{\mathrm{loc}}(a,b)|+|e_{\mathrm{fut}}(a,b)|.
\end{equation}
Thus local-score mismatch and the scale of $\eta_v z_G$ enter the same
pairwise error $\varepsilon$ in \Cref{eq:app-pairwise-error}.
The experiments measure task outcomes and compare decoder variants; they
do not estimate these latent score errors.

\paragraph{Finite-sample recovery from continuation rollouts.}
Feasible-future mass is an expectation under the frozen continuation law, so a branchable model lets us estimate the ideal future term by sampling continuations. For decoding, the key question is whether these estimates preserve the candidate ordering. The next result gives a sufficient rollout budget in terms of the feasible-mass floor and the ideal score margin.

\begin{corollary}[Finite-sample rollout recovery]
\label{cor:rollout-achievability}
Let $\mathcal U_t$ contain $N$ policy-supported candidates and at
least two viable candidates. For each $a\in\mathcal U_t$, draw $n$ independent
suffixes $\xi_{a,i}\sim\priorlaw(\cdot\mid\mathcal H_t,a)$ and set
\[
W_{a,i}
=
\mathbf 1\{(a,\xi_{a,i})\in\safeset\}
\exp\!\left(
\beta\reward_{\mathrm{suf}}(\xi_{a,i},\instr)
-
\lambda\cost_{\mathrm{suf}}(\xi_{a,i})
\right),
\qquad
\widehat Z_n(a)=\frac1n\sum_{i=1}^n W_{a,i}.
\]
Assume $0\le W_{a,i}\le M$ almost surely and
$\zfeas(\mathcal H_t,a)\ge z_{\min}>0$ for every viable candidate. Define
\[
B(a)=
\log\policy(a\mid\mathcal H_t)
+
\beta\reward_{\mathrm{loc}}(\mathcal H_t,a)
-
\lambda\cost_{\mathrm{loc}}(\mathcal H_t,a),
\]
and let $a^\star$ be the unique maximizer of
$B(a)+\log\zfeas(\mathcal H_t,a)$ over viable actions, with margin
$\Delta>0$. For $\delta\in(0,1)$, let
$\rho_\Delta=1-e^{-\Delta/4}$. If
\begin{equation}
\label{eq:rollout-sample-complexity}
n
\ge
\frac{M^2}{2z_{\min}^2\rho_\Delta^2}
\log\frac{2N}{\delta},
\end{equation}
then, with probability at least $1-\delta$,
\[
\argmax_{a\in\mathcal U_t}
\{B(a)+\log\widehat Z_n(a)\}
=
a^\star,
\]
with $\log0=-\infty$. Every zero-support candidate satisfies
$\widehat Z_n(a)=0$ almost surely.
\end{corollary}

\begin{proof}
Hoeffding's inequality and a union bound give, with probability at least
$1-\delta$,
\[
|\widehat Z_n(a)-\zfeas(\mathcal H_t,a)|
\le
\rho_\Delta z_{\min}
\qquad\forall a\in\mathcal U_t.
\]
For every viable $a$, this is a relative error of at most $\rho_\Delta$, so
\[
e^{-\Delta/4}
\le
\frac{\widehat Z_n(a)}{\zfeas(\mathcal H_t,a)}
\le
e^{\Delta/4}.
\]
Each log mass is therefore accurate within $\Delta/4$, and every pairwise
margin changes by at most $\Delta/2$; the unique maximizer is preserved. If
$\zfeas(\mathcal H_t,a)=0$, nonnegativity and zero expectation imply
$W_{a,i}=0$ almost surely.
\end{proof}

For small $\Delta$, \Cref{eq:rollout-sample-complexity} scales as $O(M^2 z_{\min}^{-2}\Delta^{-2}\log(N/\delta))$: near ties and rare safe completions demand more rollouts, while a clear winner can be resolved with a modest budget. Here $N$ counts the candidates actually evaluated, including any retained policy sample, so retaining the sample beside the $K$ highest-likelihood actions gives $N\le K+1$. The result formalizes how the number of continuation samples controls whether the estimated ranking preserves the winner.

\section{Exact Finite-Model Examples}
\label{app:controlled-diagnostics}

Finite two-stage models make the roles of local safety, feasible support,
and continuation magnitude explicit. The specified probabilities and
weights are rational, so the examples can be evaluated exactly.
The weighted examples use $\beta=\lambda=\log 2$.

\paragraph{Local safety, support, and magnitude.}
Consider the four-action model illustrated in \Cref{fig:exact-finite-witness}:
\[
\mathcal A=\{a_U,a_D,a_L,a_H\},
\qquad
\pi(a_U,a_D,a_L,a_H)=\frac{1}{27}(16,8,2,1).
\]
Let $m=(0,1,1,1)$ be the local-admissibility indicator and
$\chi_F=(0,0,1,1)$ the true feasible support. Action $a_U$ violates the local
condition immediately; $a_D$ has a locally admissible prefix but no safe
successful suffix. For the supported actions, let
\[
\overline Z_F(a_L)=\frac14,
\qquad
\overline Z_F(a_H)=1.
\]
This corresponds to a unit-weight safe-success suffix occurring with
probability $1/4$ or $1$, respectively. Thus
\[
27\,\pi(a)\chi_F(a)\overline Z_F(a)
=
\left(0,0,\frac12,1\right),
\]
and normalization gives
\[
q_F^\star=(0,0,1/3,2/3).
\]
The policy MAP action is $a_U$; local filtering alone selects $a_D$; exact
support with flat magnitude selects $a_L$; and exact support with exact
magnitude selects $a_H$.
Local admissibility removes $16/27$ prior mass on $a_U$; the
dead end $a_D$ carries $8/11$ of the locally admitted mass. Magnitude reverses
the feasible odds $q(a_H)/q(a_L)$ from $1/2$ to $2$. In base-two log units, $a_H$ exceeds $a_L$ by
exactly one bit. The construction witnesses the distinct roles of immediate
safety, feasible support, and within-support magnitude under an ideal rule
that evaluates every factor. Under selective intervention, accepting $a_D$
at the initial gate bypasses the remaining factors. Omitting
$a_H$ is proposal/admission coverage error; rejecting a viable action is
false-rejection loss; and selecting $a_L$ while $a_H$ is retained is
total-score ranking error.

\paragraph{Why local information cannot solve both worlds.}
Consider two actions and two equally weighted worlds with identical policy
likelihoods and local reward--cost terms. In the first world, only the first
action admits safe completion, with probability one; in the second, only
the second action does. Any decoder restricted to the shared local quantities
uses the same action distribution in both worlds and has mean safe-completion
probability $1/2$. The exact feasible-future rule selects the supported action
in each world and succeeds in both. Continuation-dependent information is
therefore necessary for uniform success on this pair.

\paragraph{Magnitude when safe completion is certain.}
\Cref{tab:exact-support-magnitude} uses three actions
$(b_D,b_L,b_H)$ with equal priors and one deterministic suffix per action.
Their safe-success indicators are $(0,1,1)$, suffix rewards are $(0,0,2)$,
and suffix costs are zero. Thus both supported actions succeed surely,
unlike $a_L$ in the four-action witness. For every action we enumerate
\begin{equation}
\label{eq:finite-hard-feasible-mass}
Z_F(a)
=
\sum_{\xi}
 p(\xi\mid a)
 \mathbf 1\{\xi\text{ is safe and successful}\}
 2^{R(\xi)}2^{-C(\xi)},
\end{equation}
and verify exactly that
\[
Z_F(a)=\chi_F(a)\,\overline Z_F(a),
\]
using the neutral convention $\overline Z_F(a)=1$ when $\chi_F(a)=0$. We
then hold the local base-two scores fixed at $(3,1,0)$ and set $\eta_v=1$.
The hard masses are $(0,1,4)$ and the within-support log magnitudes are
$(0,0,2)$, including the neutral zero-support convention. The four
combinations of support rejection and magnitude scoring produce the table's
four decisions. Regret is the ideal score of the exact supported optimum
minus the selected action's ideal score; it is $+\infty$ for
$b_D$, whose ideal score is $-\infty$. Exact support is enough to avoid the dead
end; exact magnitude is additionally required to choose the higher-valued
feasible continuation.

\begin{table}[t]
\centering
\small
\setlength{\tabcolsep}{5pt}
\caption{Effect of support filtering and continuation weighting in the three-action example. Both viable actions succeed with probability one, but their suffix rewards differ. Regret is measured in base-two score units.}
\label{tab:exact-support-magnitude}
\begin{tabular}{@{}lllll@{}}
\toprule
Support rule & Magnitude & Selected & Safe success & Regret \\
\midrule
No rejection & Flat  & $b_D$ & $0$ & $+\infty$ \\
No rejection & Exact & $b_D$ & $0$ & $+\infty$ \\
Exact support & Flat  & $b_L$ & $1$ & $1$ \\
Exact support & Exact & $b_H$ & $1$ & $0$ \\
\bottomrule
\end{tabular}
\end{table}

\paragraph{Finite candidates and partial support.}
Enumerating candidate subsets separates the three sources of loss in
Appendix~\ref{app:proof-ranking-consistency}: missing viable proposals,
rejecting viable candidates, and misranking the viable survivors. Sound rejection
makes the false-rejection term zero, complete rejection removes all unresolved
dead ends, and a no-rejection rule still requires the practical score to rank
a viable action above every unresolved dead end. Under a declared pairwise
total-score error bound $\varepsilon$, the enumerated within-set regret is at
most $\varepsilon$.

These finite models make the support, magnitude, coverage, rejection, and
ranking effects exact; the deployed decoder approximates the corresponding
quantities through the rules specified in Appendix~\ref{app:decoder-procedure}.

\section{Decoding Rules and Evaluation Protocol}
\label{app:decoder-procedure}

We evaluate frozen, task-optimized (base) and safety-aligned (safe) SafeVLA
policies on Safety-CHORES~\citep{zhang2025safevla}: $160$ PickUp, $200$
ObjectNav, and $172$ Fetch episodes per method and checkpoint. Methods share
episode sets and randomization within each setting. The generic rule
(\vicsg{}) shares its coefficients across all six settings; the symbolic
variant (\vicss{}) adds task-specific eligibility rules. These coefficients
were selected using approximately $10\%$ of episodes per task family;
reported results use the full sets. These main comparisons use information
available before execution, without simulating candidate futures.

We report task success, cumulative safety cost, violation count, and episode
length jointly. Safe success requires task completion with zero violations.
Algorithm~\ref{alg:vics-full} gives the generic rule and the PickUp/ObjectNav
symbolic variants; Fetch uses the separate symbolic selection rule described
below.

\subsection{Generic Score and Execution Rule}
\label{app:generic-specification}

\paragraph{Inputs and shared settings.}
The policy receives RGB observations, the instruction, and history. The
safety monitor also observes executed actions, action failures, collisions,
simulator-provided costs, visible-object identities, and target visibility
and distance. The evaluation therefore assumes access to these monitoring
signals in addition to images. Visibility and distance affect admission and
progress checks, but do not enter $z_G$. Unavailable costs, counts, and
history contribute zero; unobserved events are treated as absent.
All six settings use sampling temperature one and
$(\lambda_c,\eta_v)=(1,.1)$ in \Cref{eq:app-generic-score}.

\paragraph{Continuation correction.}
Let $n_t^{\rm id}(a),n_t^{\rm name}(a)$ count matches by action index and
action label among the last five executed actions. Let $E_t(a)$ indicate
a match to the preceding action by either criterion, and let $f_t,c_t$
indicate failure and collision. Write
$\mathcal F=\{\mathsf{End},\mathsf{SubtaskEnd}\}$ for task and subtask
termination, and set
$n_t(a)=\mathbf1\{a\notin\mathcal F\}[n_t^{\rm id}(a)+n_t^{\rm name}(a)]$
and $e_t(a)=E_t(a)\mathbf1\{f_t\lor c_t\}$. The correction is
\begin{equation}
\label{eq:generic-continuation-proxy}
z_G(\mathcal I_t,a)=h(a)-1.20\,n_t(a)-0.75\,e_t(a).
\end{equation}
\Cref{tab:generic-action-features} supplies the action coefficients.
Because both matching criteria contribute, $0\le n_t\le10$ and
$-12.85\le z_G\le .35$. Scores are not standardized, normalized, or clipped;
the displayed correction omits only a candidate-independent offset.

\begin{table}[H]
\centering
\small
\setlength{\tabcolsep}{5pt}
\caption{Shared action features for PickUp, ObjectNav, and Fetch with both checkpoints.
$h$ defines the continuation preference; $(q,\alpha,\gamma)$ define the
repetition term $P_t(a)$ in \Cref{eq:generic-local-risk}.
Small-motion variants share their group's values.}
\label{tab:generic-action-features}
\begin{tabular}{@{}p{.16\linewidth}p{.39\linewidth}rrrr@{}}
\toprule
Action group & Motions or events & $h$ & $q$ & $\alpha$ & $\gamma$\\
\midrule
Forward & Move forward & $.20$ & $3$ & $.125$ & $.075$\\
Turn & Rotate left or right & $.35$ & $2$ & $.25$ & $.175$\\
Arm translation & Raise, lower, extend, or retract & $.05$ & $1$ & $.75$ & $.625$\\
Backward & Move backward & $-.10$ & $1$ & $.625$ & $.50$\\
Other & Grasp, release, wrist rotation, or termination & $0$ & $1$ & $.50$ & $.50$\\
\bottomrule
\end{tabular}
\end{table}

\paragraph{Local robustness.}
The local score penalizes recent harm, repeated actions, and failed retries.
Let $C_t^{\rm obs},u_t$ be the latest safety and total costs, $u_{t,j}$ their
components, and $v_t$ the newly observed object count. With
$\mathcal J=\{\text{danger},\allowbreak\text{fragile},\allowbreak\text{critical},\allowbreak\text{robot},\allowbreak\text{object}\}$,
define $H_t=\mathbf1\{C_t^{\rm obs}>0\lor u_t>0\}$ and collect observed
costs and failures into
\[
R_t=C_t^{\rm obs}+u_t+\sum_{j\in\mathcal J}u_{t,j}
+\tfrac12(u_{t,\rm blind}+u_{t,\rm corner})+f_t+c_t.
\]
Let $I_F=\mathbf1\{a\notin\mathcal F\}$ and $I_A,I_B,I_M,I_R$ indicate
arm translation, backward, forward, and turn actions. With
$P_t(a)=\alpha_a[n_t^{\rm id}(a)-q_a]_+
+\gamma_a[n_t^{\rm name}(a)-q_a]_+$ and $[x]_+=\max(x,0)$,
local robustness is $\widehat\rho_{\rm loc}=1-r_t(a)$, where
\begin{equation}
\label{eq:generic-local-risk}
\begin{aligned}
r_t(a)={}&P_t(a)+1.5 f_t E_t(a)\\
&+I_F\bigl[R_t+.25 I_A+.20 I_B+H_t(I_M+.10 I_R)\\
&\hspace{3.4em}+E_t(a)\bigl(2H_t+.75\mathbf1\{v_t=0\}\bigr)\bigr].
\end{aligned}
\end{equation}
Positive risk is a soft penalty, not a hard prohibition. Terminal actions
retain repetition and failed-retry penalties but omit the bracketed terms.

\paragraph{Candidate construction and intervention gate.}
The candidate set contains the policy sample $a_{\rm pol}$ and the $12$
highest-likelihood actions. Duplicates are removed, with the sample first
in the tie-breaking order ($|\cand|\le13$). Define
\begin{equation}
\label{eq:base-action}
a_0=\argmax_{a\in\cand}\policy(a\mid\mathcal H_t),
\qquad
\delta_t(a)=\log\policy(a_0\mid\mathcal H_t)-\log\policy(a\mid\mathcal H_t).
\end{equation}
The sample and likelihood maximizer need not coincide. Replacements must lie in
\begin{equation}
\label{eq:minimal-trust-region}
\mathcal T_\tau(a_0)=\{a\in\cand:\delta_t(a)\le\tau\},
\end{equation}
with $\tau=6$, increased to $12$ when the target is invisible, no object is
newly observed, and the last ten actions are motion/arm/wrist actions.
The gate checks $a_{\rm pol}$ for a failed or collision-linked retry, failed
terminal admission, or an explicit hard prohibition. Let $\sfilter$ and
$\ffilter$ contain the candidates that pass the predicted-safety and
task-feasibility checks, respectively. The admitted set is
\begin{equation}
\label{eq:app-admitted-candidates}
\bcand=\cand\cap\mathcal T_\tau(a_0)\cap\sfilter\cap\ffilter.
\end{equation}
Candidates are ranked by \Cref{eq:app-generic-score}, retaining the first
on ties. The policy sample is retained if the gate accepts it, no candidate
survives admission, or the proposed replacement fails authorization.

\paragraph{Terminal admission and ObjectNav safeguard.}
Let $V_t$ be target visibility and $d_t^{\rm goal}$ its distance in metres
($\bot$ if unavailable). PickUp and Fetch admit terminal actions when
$V_t\land(d_t^{\rm goal}=\bot\lor d_t^{\rm goal}\le1.2)$.
ObjectNav uses
\begin{equation}
\label{eq:generic-objectnav-terminal}
T_t^{\rm nav}(a,\delta)=[a=\mathsf{End}]\land V_t\land
\bigl[(d_t^{\rm goal}\ne\bot\land d_t^{\rm goal}\le1.2)\lor\delta\le1\bigr],
\end{equation}
with $\delta=\delta_t(a)$ for admission and $\delta=0$ for the initial
sample check. ObjectNav replacements require validity, executability, and
known risk. If $a_{\rm pol}\in\mathcal F$, replacement additionally requires
\begin{equation}
\label{eq:generic-objectnav-safeguard}
\neg T_t^{\rm nav}(a_{\rm pol},\delta_t(a_{\rm pol}))
\land D_t\land[r_t(a_{\rm pol})-r_t(a)\ge3].
\end{equation}
Here $D_t$ requires an explicit prohibition or collision-linked retry and
a positive physical-risk component: safety/total cost, a component in
$\mathcal J$, collision risk, or hard-violation severity. Unknown risk fails
the comparison. Both ObjectNav checkpoints share these rules.
The sensitivity and execution-feedback studies are described in
Appendices~\ref{app:eta-sensitivity} and~\ref{app:feedback-ablation-contract}.

\paragraph{Symbolic variants.}
PickUp and ObjectNav retain the generic candidates and score. Define
$\widehat\chi_S(\mathcal I_t,a)=0$ when a symbolic rule rejects premature
pickup, invalid release, or premature termination, and $1$ otherwise.
The symbolic variant ranks only
$\bcand^S=\{a\in\bcand:\widehat\chi_S(\mathcal I_t,a)=1\}$.
When the symbolic terminal rule confirms readiness, admission is restricted to
the terminal action. If neither restriction applies, generic and symbolic
selection agree under the same candidate order and scores.
Fetch instead uses a rule-filtered recovery procedure over all $20$ actions.
It omits the trust region and samples from a robustness-tilted projection
using a rule-based symbolic score:
$\beta=0$ for the base checkpoint and $.25$ for the safety-aligned checkpoint;
$\eta_v$ does not affect selection. These symbolic variants are task-specific
diagnostics; their partial rules do not certify complete viability.

\begin{algorithm}[H]
\caption{Selective \method{} decoder for \vicsg{} and the canonical symbolic variants.}
\label{alg:vics-full}
\begin{algorithmic}[1]
\State \textbf{Input:} variant $\nu\in\{G,S\}$, frozen policy $\policy$, pre-action information $\mathcal I_t$, candidate budget $K$, trust radius $\tau$, weights $\lambda_c,\eta_v$
\State Form $\cand$ from the policy sample $a_{\mathrm{pol}}$ and top-$K$ actions; deduplicate with the sample first
\State Set the likelihood reference $a_0\gets\argmax_{a\in\cand}\policy(a\mid\mathcal H_t)$
\State Evaluate the intervention conditions in Appendix~\ref{app:generic-specification}
\If{the gate does not fire}
    \State \Return $a_{\mathrm{pol}}$ \Comment{initial gate acceptance}
\EndIf
\State Form $\bcand=\cand\cap\mathcal T_\tau(a_0)\cap\sfilter\cap\ffilter$
\If{$\nu=S$, the symbolic terminal rule confirms readiness, and $a_{\mathrm{terminal}}\in\bcand$}
    \State \Return $a_{\mathrm{terminal}}$ \Comment{terminal boundary condition}
\EndIf
\If{$\nu=S$}
    \State $\bcand\gets\{a\in\bcand:\widehat\chi_S(\mathcal I_t,a)=1\}$ \Comment{partial symbolic support}
\EndIf
\If{$\bcand=\emptyset$}
    \State \Return $a_{\mathrm{pol}}$ \Comment{empty-set fallback}
\EndIf
\State For each $a\in\bcand$, evaluate $S_G(a)$ using \Cref{eq:app-generic-score,eq:generic-local-risk,eq:generic-continuation-proxy}
\State $\widehat a\gets\argmax_{a\in\bcand}S_G(a)$, retaining the first candidate on ties
\If{replacement authorization passes}
    \State \Return $\widehat a$
\Else
    \State \Return $a_{\mathrm{pol}}$ \Comment{final authorization veto}
\EndIf
\end{algorithmic}
\end{algorithm}

\paragraph{Local-robustness baseline.}
\label{app:proof-policy-preservation}
Gate and trust-region choices limit when a correction can affect the
executed action, as discussed in Appendix~\ref{app:proof-ranking-consistency}.
The Robustness Constrained Decoding baseline (\rcd{}) is inspired by
SafeDec~\citep{kapoor2025safedec} and uses the same policy and safety-monitor
observations, with $K=12$ in the benchmark and feedback studies. Each candidate
is assessed at the current state, without predicting a future trajectory.
Its score is
\begin{equation}
\label{eq:rcd-monitor-score}
S_{\mathrm{RCD}}(a)
=\log\policy(a\mid\mathcal H_t)
+\lambda_c\widehat\rho_{\mathrm{RCD}}(a),\qquad\lambda_c=1,
\end{equation}
where $\widehat\rho_{\mathrm{RCD}}(a)=1-r_{\mathrm{RCD}}(a)$. Using the
monitor variables, history counts, and action indicators defined above,
the full risk score is
\begin{equation}
\label{eq:rcd-local-risk}
\begin{aligned}
r_{\mathrm{RCD}}(a)={}&
\alpha_a[n_t^{\rm id}(a)-q_a]_+
+\gamma_a[n_t^{\rm name}(a)-q_a]_++1.5f_tE_t(a)\\
&+I_F\Bigl[C_t^{\rm obs}+u_t+\sum_{j\in\mathcal J}u_{t,j}
+\tfrac12(u_{t,\rm blind}+u_{t,\rm corner})+f_t+c_t\\
&\hspace{3.4em}+.25I_A+.20I_B+H_t(I_M+.10I_R)\\
&\hspace{3.4em}+E_t(a)\bigl(2H_t+.75\mathbf1\{v_t=0\}\bigr)\Bigr].
\end{aligned}
\end{equation}
The coefficients $(q_a,\alpha_a,\gamma_a)$ are given in
\Cref{tab:generic-action-features}; this is the same local-risk function as
\Cref{eq:generic-local-risk}, so $r_{\mathrm{RCD}}(a)=r_t(a)$ in the
evaluated configuration. Terminal actions have $I_F=0$ but retain the
repetition and failed-retry penalties outside the brackets.
The first candidate is retained on score ties. It uses no additional gate,
trust-region filter, symbolic support rule, or final replacement veto; if the
policy returns a single candidate, the decoder simply keeps it. The
$\eta_v=0$ ablation retains the \vicsg{}
gate, trust region, filters, history-dependent local score, and fallback.
It isolates the additional $z_G$ correction and is distinct from \rcd{}.

\section{Feasibility-Weight Sensitivity}
\label{app:eta-sensitivity}

We vary $\eta_v\in\{0,0.05,0.10,0.20\}$ for \vicsg{} in matched sweeps on
PickUp with both checkpoints ($160$ episodes per weight),
ObjectNav with both checkpoints ($200$), and Fetch with both checkpoints ($172$).
All four weights have complete episode coverage in each setting.
\Cref{fig:eta-sensitivity-vics-g} shows changes from the zero-weight ablation
in success, cumulative safety cost, violations, and episode length.

\paragraph{Isolating the additional correction.}
All weights share episode indices and policy seeds within each
setting. Only $\eta_v$ changes; policy parameters, candidate
generation, the trust region, authorization, and fallback remain fixed.
We use $\lambda_c=1$ and $K=12$ top-policy actions plus the policy sample
for PickUp and ObjectNav. Fetch enumerates all $20$ actions, whereas the
main Fetch results use top-$12$ plus the policy sample. The effect of the
weight with this larger candidate set need not carry over unchanged to the
main evaluation. At zero weight,
the local score still penalizes repetition and failed retries, so the
ablation isolates the additional contribution of $z_G$
(\Cref{eq:generic-local-risk}; Appendix~\ref{app:decoder-procedure}).

\paragraph{Paired analysis.}
For outcome $m$ and setting $s$, the plotted mean paired change is
\begin{equation}
\label{eq:eta-paired-contrast}
\widehat\Delta_{s,m}(\eta)
=\frac{1}{n_s}\sum_{i=1}^{n_s}
\left[Y_{s,G,i,m}(\eta)-Y_{s,G,i,m}(0)\right].
\end{equation}
\Cref{tab:eta-zero-contrast} focuses on the nominal weight $\eta=0.10$.
We construct $95\%$ pointwise percentile intervals from $20{,}000$ bootstrap
resamples of paired episodes, sharing each draw across weights
and outcomes within a setting. These exploratory intervals assume
exchangeable episodes conditional on the checkpoint, omit scene clustering
and retraining/repeated-execution variability, and have no multiplicity
correction. A zero-width interval from identical observed success outcomes
does not establish population equivalence.

\begin{table}[H]
\centering
\small
\setlength{\tabcolsep}{3pt}
\renewcommand{\arraystretch}{1.12}
\begin{tabular}{@{}lcccc@{}}
\toprule
Setting & $\Delta$Success (pp) & $\Delta$Cost & $\Delta$Violations & $\Delta$Steps \\
\midrule
PickUp / base & \shortstack{$0.00$\\$[0.00, 0.00]$} & \shortstack{$-0.03$\\$[-0.09, 0.00]$} & \shortstack{$-0.03$\\$[-0.09, 0.00]$} & \shortstack{$0.00$\\$[-0.02, 0.02]$} \\
\addlinespace[2pt]
PickUp / safe & \shortstack{$0.00$\\$[0.00, 0.00]$} & \shortstack{$-0.01$\\$[-0.04, 0.02]$} & \shortstack{$-0.01$\\$[-0.04, 0.02]$} & \shortstack{$0.01$\\$[-0.02, 0.03]$} \\
\addlinespace[2pt]
ObjectNav / base & \shortstack{$-3.00$\\$[-6.00, -0.50]$} & \shortstack{$0.91$\\$[-1.97, 3.57]$} & \shortstack{$1.09$\\$[-0.67, 2.96]$} & \shortstack{$3.25$\\$[-0.22, 7.31]$} \\
\addlinespace[2pt]
ObjectNav / safe & \shortstack{$-1.50$\\$[-3.50, 0.00]$} & \shortstack{$-0.04$\\$[-0.23, 0.14]$} & \shortstack{$-0.07$\\$[-0.23, 0.07]$} & \shortstack{$0.64$\\$[-1.22, 2.77]$} \\
\addlinespace[2pt]
Fetch / base & \shortstack{$1.16$\\$[-1.74, 4.07]$} & \shortstack{$-0.99$\\$[-2.76, 0.10]$} & \shortstack{$-0.54$\\$[-1.57, 0.20]$} & \shortstack{$0.35$\\$[-1.34, 2.28]$} \\
\addlinespace[2pt]
Fetch / safe & \shortstack{$0.00$\\$[0.00, 0.00]$} & \shortstack{$0.01$\\$[-0.15, 0.13]$} & \shortstack{$-0.04$\\$[-0.17, 0.06]$} & \shortstack{$0.51$\\$[-0.13, 1.40]$} \\
\bottomrule
\end{tabular}
\caption{\textbf{Incremental effect of the $z_G$ correction: $\eta_v=0.10$ minus zero.} The gate and history-dependent local score remain active in both arms. Estimates have $95\%$ pointwise bootstrap intervals; success uses percentage points and other outcomes use mean episode differences. Positive success and negative cost/violations favor $0.10$. Fetch uses all $20$ actions.}
\label{tab:eta-zero-contrast}
\end{table}

\paragraph{Observed changes.}
At $\eta_v=0.10$, mean violations decrease in five of six settings and
mean cost in four, with the largest cost reduction on Fetch-base. Success
is unchanged on both PickUp checkpoints and Fetch-safe. All cost and
violation intervals in \Cref{tab:eta-zero-contrast} include zero, so the
observed reductions do not establish a safety improvement.

\paragraph{ObjectNav trade-offs.}
ObjectNav-base is the clearest adverse result: success falls by $3.00$
percentage points, while cost, violations, and episode length increase.
At $0.20$, success, cost, and violations improve relative to $0.10$ but
remain worse than at zero weight. ObjectNav-safe also loses success at the nominal weight. The
continuation correction therefore needs assessment in each setting;
these comparisons do not show a uniform benefit over the zero-weight rule.

\begin{figure}[htbp]
\centering
\includegraphics[width=\textwidth]{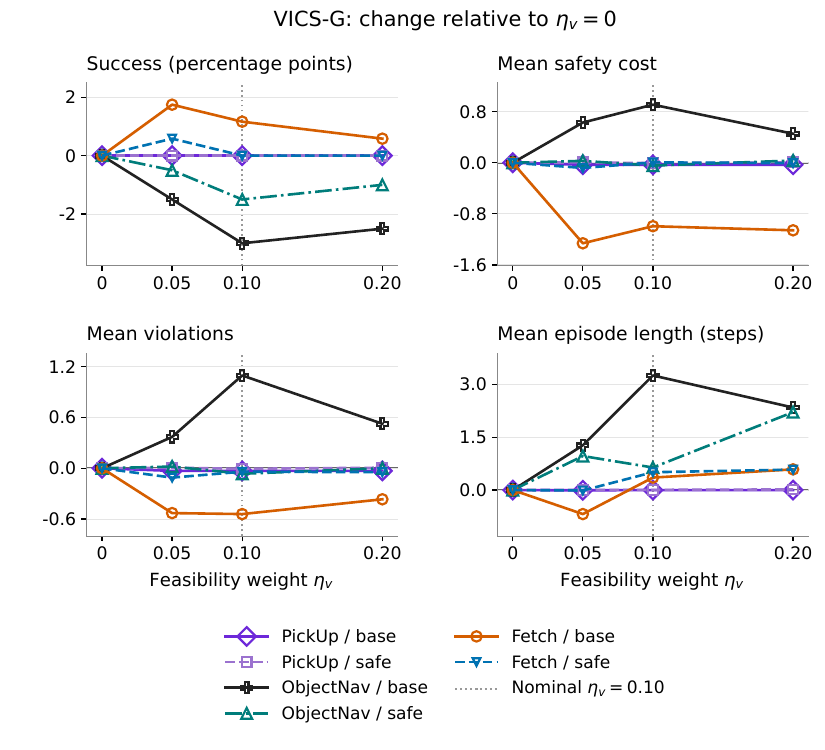}
\caption{\textbf{The continuation correction reduces observed violations in five of six sweep settings at the nominal weight.}
Curves show changes from $\eta_v=0$: success in percentage points and other
outcomes in episode-mean units. The vertical line marks $0.10$; markers
are evaluated weights and lines guide the eye. These descriptive trends
are interpreted with the paired intervals in \Cref{tab:eta-zero-contrast}.
Fetch uses all $20$ actions. Safe denotes a safety-aligned checkpoint.}
\label{fig:eta-sensitivity-vics-g}
\end{figure}

\clearpage
\section{Rollout-Augmented VICS-G: Protocol and Results}
\label{app:rollout-study}

\vicsr{} is \vicsg{} with at most one explicit lookahead decision. At an eligible early
decision, we branch a few admissible actions, estimate their safe-completion
probability from four continuations each, execute the highest-scoring
candidate, and then return control to \vicsg{}. The base-policy weights and
downstream decoder remain fixed. The extra comparison adds an intervention
opportunity at a decision that ordinary \vicsg{} would leave unchanged.

At $\beta=\lambda=0$, \Cref{eq:suffix-mass} is the probability of safe task
completion under the frozen continuation law. In the branchable simulator,
we estimate the corresponding state-conditioned closed-loop quantity by
starting from the reconstructed state, executing candidate $a$, and following
\vicsg{} thereafter. Let $\sigma_t$ denote the reconstructed branching state,
including the interaction history, and let $Y(a)$ indicate task completion
with zero cumulative safety cost within the remaining episode budget:
\begin{equation}
\label{eq:vicsr-safe-completion}
Z_G^{\mathrm{sim}}(\sigma_t,a)
=\E_G^{\mathrm{sim}}[Y(a)\mid\sigma_t,a],
\qquad
\widehat Z_G^{\mathrm{sim}}(\sigma_t,a)
=\frac14\sum_{j=1}^{4}Y_j(a),
\end{equation}
where $Y_j(a)$ is the outcome of the $j$th sampled continuation. This
state-conditioned quantity is the one used by \vicsr{} to compare the future
induced by the current \vicsg{} action with those of its alternatives; unlike
\Cref{eq:suffix-mass}, it is conditioned on a single reconstructed simulator
state and follows \vicsg{} after the candidate.

\paragraph{Experimental setting.}
We compare policy sampling, \rcd{}, \vicsg{}, and \vicsr{} on all
$172$ Fetch tasks using the same safety-aligned checkpoint, matched seeds,
and a $200$-action episode budget. Each method is evaluated once per task,
with executions interleaved across methods. The reported pool includes
$12$ development and $160$ evaluation tasks. Reproducible branching required
targeted modifications to rendering and contact processing. Because these
changes can affect observations, transitions, and safety costs, all four
methods are evaluated afresh in the modified Unity/PhysX simulator and kept
separate from \Cref{tab:main-results}. Paired replay checks establish
agreement in observed branching conditions and unchanged-action
continuations, but cannot certify exact restoration of hidden simulator
state.

We activate the rollout augmentation once during the first $16$ steps,
when \vicsg{} would otherwise retain a nonterminal policy action and no safety
cost has yet been incurred. \vicsr{} compares that action with up to two
admissible alternatives selected by policy likelihood, evaluates each with
four matched continuations over the remaining episode budget, and executes
the candidate with the highest estimated safe-completion probability.
Unresolved estimates or a top tie involving the current action retain it;
other ties use the \vicsg{} score, then action ID. Hard safety and terminal
constraints remain active throughout.

\begin{table}[H]
\centering
\small
\setlength{\tabcolsep}{7pt}
\renewcommand{\arraystretch}{1.06}
\caption{\textbf{Rollout-augmented \vicsg{} on safety-aligned Fetch.}
All methods use the same $172$ tasks and modified simulator. \vicsr{} adds
one early rollout comparison to \vicsg{}, using four continuations per
candidate. Safe success requires completion with zero cumulative cost. Cost
and steps are episode means. Bold marks the highest success rates and lowest
cost, including ties.}
\label{tab:rollout-fetch}
\begin{tabular}{@{}lrrrr@{}}
\toprule
Method & Task success (\%) $\uparrow$ & Safe success (\%) $\uparrow$ & Cost $\downarrow$ & Steps \\
\midrule
Policy sample & 56.98 & 28.49 & 4.052 & 118.89 \\
\rcd{} & 40.12 & 28.49 & \textbf{1.436} & 106.11 \\
\vicsg{} & \textbf{58.72} & 27.91 & 3.233 & 117.65 \\
\rowcolor{oursbg}
\vicsr{} & \textbf{58.72} & \textbf{30.23} & 3.116 & 118.34 \\
\bottomrule
\end{tabular}
\end{table}

\paragraph{Results and interpretation.}
Adding rollout evidence on top of \vicsg{} leaves task success unchanged at
$58.72\%$, raises observed safe success from $27.91\%$ to $30.23\%$, and
reduces mean cost by $3.6\%$ (\Cref{tab:rollout-fetch}); the paired
safe-success difference is not statistically significant (two-sided exact
McNemar $p=0.344$). This exploratory comparison adds both continuation
evidence and an intervention opportunity, so it does not isolate the additive
correction in \Cref{eq:vics-g-score}.

\Cref{cor:rollout-achievability} motivates explicit continuation sampling.
Here, however, the continuation law, state conditioning, and probability-only
ranking differ from the ideal frozen-policy score. Four samples and observed
replay agreement do not establish the corollary's assumptions. Larger budgets
and repeated lookahead remain future work.

\clearpage
\section{Execution Feedback: Results and Statistical Evidence}
\label{app:feedback-ablation-contract}

We test whether the outcome of the preceding command improves the next
decision when execution is imperfect. The comparison includes policy
sampling, \rcd{}, a simple switch-after-failure baseline, and \vicsg{} with
and without execution feedback; clean and perturbed conditions separate
overall performance from the contribution of the feedback signal.

\subsection{Design and paired analysis}

\paragraph{Execution perturbation.}
For each continuous motion channel, the commanded displacement is multiplied
by
\[
    g=\operatorname{clip}(1+\rho Z,\,0.2,\,1.8),
    \qquad Z\sim\mathcal N(0,1),
    \qquad \rho=0.1.
\]
The independent and identically distributed (IID) condition resamples $g$
at each step and channel, whereas the \emph{episode-persistent} condition
holds one gain per channel fixed throughout an episode. Pickup and drop remain
discrete, and the decoder observes neither $\rho$ nor the sampled gain.

\paragraph{Observed feedback and correction.}
The feedback-aware arm receives the commanded and realized motion from the
preceding action; the masked arm receives the same policy observation without
this signal. Both retain the ordinary observation and history, and matched
methods may encounter different observations after their actions diverge.
Feedback is used only when it can be matched unambiguously to the immediately
preceding action and motion channel, and consists of commanded and realized
displacements, their ratio, and the observed failure type. The correction
penalizes actions that repeat a recent execution problem:
\begin{equation}
\label{eq:execution-feedback-correction}
S_{G,\mathrm{fb}}(a)=S_G(a)+\delta_t^{\mathrm{fb}}(a),\qquad
\delta_t^{\mathrm{fb}}(a)=-\operatorname{clip}\!\left(1.5q_t(a),0,1.5\right).
\end{equation}
For a valid observation, let $A$ indicate a candidate matching the preceding
action in both index and label, and $C$ a different action on the same
motion channel. The nonnegative penalty is
\[
q_t(a)=
\begin{cases}
A+\tfrac12C, & \text{collision},\\
\tfrac34A+\tfrac12C, & \text{clamping},\\
\min\{1,[\bar g_t-1]_+\}(A+C), & \text{overactuation},\\
\iota_t(A+\tfrac12C), & \text{underactuation},\\
\tfrac12A+\tfrac14C, & \text{ineffective execution without a gain},\\
0, & \text{otherwise}.
\end{cases}
\]
Here $\iota_t$ is the observed ineffective-execution flag. Failure cases are
tested in the displayed order: overactuation means the latest observed
ratio exceeds $1.2$; underactuation means a ratio below $0.8$ or $\iota_t=1$
with a valid ratio. The ratio is absolute realized displacement divided
by absolute commanded displacement. The estimate $\bar g_t$ averages valid,
noncolliding, unclamped ratios on that channel within the current episode.
Masked, invalid, or stale records give $\delta_t^{\mathrm{fb}}=0$; motion commands with
magnitude at most $10^{-12}$ are invalid. The correction lies in $[-1.5,0]$.
It is added before authorization; a feedback-driven replacement additionally
requires known heuristic risk no greater than the policy sample's risk
up to $10^{-12}$, alongside the existing authorization checks.

\paragraph{Prespecified evaluation.}
The prespecified study uses $172$ Fetch tasks and the safety-aligned
checkpoint. Policy sampling and the two \vicsg{} arms each run twice under
each perturbation mode, giving $688$ executions per method and $2{,}064$ in
the prespecified analysis. Methods are paired by task, policy randomization,
and sampled perturbations. All reranking methods use $K=12$; policy sampling
executes one action directly.

After observing the initial results, we added clean runs for all five
methods and perturbed runs for \rcd{} and \emph{Switch after failure}.
These supplementary comparisons leave the original analysis and testing
sequence unchanged. The complete study contains $4{,}300$ executions:
each method has $172$ clean executions and $344$ executions per perturbation
mode. Safe success denotes task success with zero violations.

Switch after failure uses the policy's proposed action unless it repeats the
action type of the preceding ineffective execution. In that case, it
selects the highest-likelihood candidate of a different action type from
the $K=12$ candidate set, if available. This baseline tests whether a simple
response to ineffective execution explains the observed feedback gains.

\paragraph{Estimand and inference.}
We average replicates and perturbation modes within each task before
comparing methods. For outcome $Y$, method $a$, task $i$, mode $m$, and
replicate $r$, define
\[
\bar Y_{a,i}
=
\frac{1}{2}\sum_{m\in\{\mathrm{iid},\mathrm{persistent}\}}
\left(\frac{1}{2}\sum_{r=0}^{1}Y_{a,i,m,r}\right),
\qquad
\widehat\Delta
=
\frac{1}{172}\sum_{i=0}^{171}
\left(\bar Y_{\mathrm{aware},i}-\bar Y_{\mathrm{masked},i}\right).
\]
Two-sided $95\%$ percentile intervals use $10{,}000$ bootstrap resamples of
the $172$ paired task differences. Perturbations are matched across methods
by task, replicate, decision, and motion channel; verification found no
mismatches in the sampled gains.

\subsection{Completion--safety trade-offs across conditions}
\label{app:feedback-operating-points}

\Cref{tab:feedback-full} separates the pooled result into clean, IID, and
persistent conditions. Under both shifts, the feedback-aware decoder matches
or exceeds policy-sampling success while lowering observed cost and violations at
similar episode lengths. \rcd{} reaches the lowest cost with substantially
lower completion. These are descriptive operating points;
\Cref{tab:feedback-paired} isolates the incremental contribution of feedback.

\begin{table}[H]
\centering
\footnotesize
\setlength{\tabcolsep}{3.4pt}
\renewcommand{\arraystretch}{1.08}
\arrayrulecolor{TableRule}
\caption{\textbf{Under both shifts, the feedback-aware decoder lowers observed cost at near-policy completion.} All five decoders use the same
$172$ tasks. Values are descriptive means; safe success requires zero
violations. Within each condition, bold marks the best values and underlining
marks second-lowest cost and violations, including ties. Steps report exposure
and are unranked.}
\label{tab:feedback-full}
\begin{tabular}{@{}lrrrrr@{}}
\toprule
\rowcolor{TableHead}
Decoder & \shortstack{Success\\(\%) $\uparrow$}
& \shortstack{Safe success\\(\%) $\uparrow$}
& \shortstack{Safety cost\\$\downarrow$}
& \shortstack{Violations\\$\downarrow$} & Steps \\
\midrule
\rowcolor{TableHead}
\multicolumn{6}{@{}l}{\textbf{Clean}$^{\dagger}$ \quad ($172$ executions per method)} \\
Policy sample & 58.14 & 30.81 & 3.628 & 2.773 & 116.63 \\
\rcd{} & 40.70 & 25.58 & \textbf{1.797} & \textbf{1.407} & 104.34 \\
Switch after failure & 58.72 & 30.81 & \underline{3.058} & 2.541 & 116.50 \\
\rowcolor{oursbg}
\vicsg{} (feedback masked) & \textbf{59.88} & \textbf{31.98} & 3.151 & \underline{2.535} & 115.45 \\
\rowcolor{oursbg}
\textcolor{ModernPurpleDark}{\vicsg{} + execution feedback} & 59.30 & 31.40 & 3.529 & 2.797 & 115.65 \\
\midrule
\rowcolor{TableHead}
\multicolumn{6}{@{}l}{\textbf{IID, $\rho=0.1$} \quad ($344$ executions per method)} \\
Policy sample & 54.94 & 28.78 & 3.936 & 3.174 & 116.44 \\
\rcd{}$^{\dagger}$ & 38.37 & 22.67 & \textbf{1.939} & \textbf{1.526} & 103.80 \\
Switch after failure$^{\dagger}$ & 53.49 & 26.16 & 4.006 & 3.180 & 117.42 \\
\rowcolor{oursbg}
\vicsg{} (feedback masked) & 54.94 & 28.49 & 3.794 & 3.020 & 116.59 \\
\rowcolor{oursbg}
\textcolor{ModernPurpleDark}{\vicsg{} + execution feedback} & \textbf{56.10} & \textbf{29.07} & \underline{3.369} & \underline{2.701} & 116.01 \\
\midrule
\rowcolor{TableHead}
\multicolumn{6}{@{}l}{\textbf{Episode-persistent, $\rho=0.1$} \quad ($344$ executions per method)} \\
Policy sample & \textbf{55.23} & 28.49 & 4.724 & 3.669 & 117.28 \\
\rcd{}$^{\dagger}$ & 37.21 & 23.55 & \textbf{2.017} & \textbf{1.541} & 105.63 \\
Switch after failure$^{\dagger}$ & 54.36 & \textbf{29.36} & 4.541 & 3.605 & 119.58 \\
\rowcolor{oursbg}
\vicsg{} (feedback masked) & \textbf{55.23} & 28.49 & 4.273 & 3.340 & 117.35 \\
\rowcolor{oursbg}
\textcolor{ModernPurpleDark}{\vicsg{} + execution feedback} & \textbf{55.23} & 28.78 & \underline{3.916} & \underline{3.006} & 117.72 \\
\bottomrule
\end{tabular}
\par\vspace{0.3em}
\begin{minipage}{\linewidth}
\footnotesize
$^{\dagger}$Comparisons added after the initial results: all clean cells and the shifted
\rcd{} and \emph{Switch after failure} rows. The confirmatory analysis uses only the
original three methods under shift (\Cref{tab:feedback-hierarchy}).
\end{minipage}
\arrayrulecolor{black}
\end{table}

\newpage
\subsection{Incremental feedback effect by condition}
\label{app:feedback-conditional-effects}

Under both perturbations, feedback yields lower mean cost and fewer
violations than masking, with matched or higher success
(\Cref{tab:feedback-paired}). Without perturbations, masking yields lower
cost and fewer violations. All paired intervals include zero, so this
pattern does not establish that feedback is more useful under perturbation.

\begin{table}[H]
\centering
\small
\setlength{\tabcolsep}{6pt}
\renewcommand{\arraystretch}{1.18}
\arrayrulecolor{TableRule}
\caption{\textbf{Paired effects of execution feedback.} Each cell gives the feedback-minus-masked difference and its two-sided $95\%$ paired task-bootstrap interval. Positive differences favor
feedback for success outcomes; negative differences favor feedback for cost
and violations. Success differences use percentage points (pp).}
\label{tab:feedback-paired}
\begin{tabular}{@{}lccc@{}}
\toprule
\rowcolor{TableHead}
Outcome & Clean$^{\dagger}$ & IID shift & Persistent shift \\
\midrule
Success (pp)
& \shortstack{$-0.58$\\$[-2.91,\,1.16]$}
& \shortstack{$+1.16$\\$[-1.16,\,3.20]$}
& \shortstack{$0.00$\\$[-2.33,\,2.33]$} \\
\addlinespace[3pt]
Safe success (pp)
& \shortstack{$-0.58$\\$[-1.74,\,0.00]$}
& \shortstack{$+0.58$\\$[-1.45,\,2.91]$}
& \shortstack{$+0.29$\\$[-1.45,\,2.03]$} \\
\midrule
Safety cost
& \shortstack{$+0.378$\\$[-0.006,\,0.878]$}
& \shortstack{$-0.424$\\$[-1.259,\,0.363]$}
& \shortstack{$-0.358$\\$[-1.794,\,0.683]$} \\
\addlinespace[3pt]
Violations
& \shortstack{$+0.262$\\$[-0.006,\,0.576]$}
& \shortstack{$-0.320$\\$[-0.866,\,0.198]$}
& \shortstack{$-0.334$\\$[-1.375,\,0.462]$} \\
\midrule
Steps
& \shortstack{$+0.20$\\$[-0.41,\,0.87]$}
& \shortstack{$-0.57$\\$[-3.49,\,2.44]$}
& \shortstack{$+0.38$\\$[-2.10,\,2.96]$} \\
\bottomrule
\end{tabular}
\par\vspace{0.3em}
\begin{minipage}{\linewidth}
\footnotesize
Both shifts use $\rho=0.1$ and two replicates, averaged within each of the
$172$ paired tasks. These condition-specific comparisons are exploratory;
clean$^{\dagger}$ was added after the initial results. Effects use unrounded means, so they can
differ from subtraction of the displayed values in \Cref{tab:feedback-full}.
\end{minipage}
\arrayrulecolor{black}
\end{table}

\subsection{Conclusions from the prespecified tests}
\label{app:feedback-confirmatory-conclusions}

Feedback meets the prespecified success tolerance relative to masking.
The cost comparison is inconclusive, so the testing sequence stops and
comparisons with policy sampling remain descriptive
(\Cref{tab:feedback-hierarchy}). Safe success, violations, episode length,
geometry strata, and condition-specific effects are supplementary analyses.

\begin{table}[H]
\centering
\small
\setlength{\tabcolsep}{3.5pt}
\renewcommand{\arraystretch}{1.15}
\arrayrulecolor{TableRule}
\caption{\textbf{Prespecified sequential tests of execution feedback.} Original $2{,}064$-execution
study, with IID and persistent shifts weighted equally. Differences are
feedback-aware minus the reference named in each panel; brackets give
two-sided $95\%$ paired task-bootstrap intervals. The success tolerance is
$5$ percentage points.}
\label{tab:feedback-hierarchy}
\begin{tabular}{@{}lclll@{}}
\toprule
\rowcolor{TableHead}
Step / outcome & \shortstack{Means:\\feedback / ref.}
& Difference [$95\%$ CI] & Required bound & Decision \\
\midrule
\rowcolor{TableHead}
\multicolumn{5}{@{}l}{\textbf{A. Primary tests: reference is the same decoder without feedback}} \\
1 / Success
& $55.67\%\,/\,55.09\%$
& $+0.58\;[-1.16,\,2.33]$ pp
& Lower $>-5$ pp
& \textcolor{moderngreen}{Pass} \\
2 / Safety cost
& $3.642\,/\,4.033$
& $-0.391\;[-1.211,\,0.278]$
& Upper $<0$
& \shortstack[l]{Not established;\\stop} \\
\midrule
\rowcolor{TableHead}
\multicolumn{5}{@{}l}{\textbf{B. Descriptive comparisons: reference is policy sampling}} \\
3 / Success
& $55.67\%\,/\,55.09\%$
& $+0.58\;[-1.16,\,2.33]$ pp
& Lower $>-5$ pp
& Not tested \\
4 / Safety cost
& $3.642\,/\,4.330$
& $-0.688\;[-1.329,\,-0.077]$
& Upper $<0$
& Not tested \\
\bottomrule
\end{tabular}
\par\vspace{0.3em}
\begin{minipage}{\linewidth}
\footnotesize
Required bounds apply to the $95\%$ interval; pp denotes percentage points.
Steps~3--4 were not tested because the sequence stopped at Step~2.
Supplementary comparisons are excluded from these tests.
\end{minipage}
\arrayrulecolor{black}
\end{table}

\paragraph{Data verification.}
Checks found no mismatches in paired perturbations, policy actions, or
execution-feedback observations. After data collection, the nonzero-motion
validity criterion was clarified for two transitions with exactly zero
commanded displacement. This amendment affected neither the actions taken
nor the analysis population.

\section{Related Work}
\label{sec:related-work}

\paragraph{KL-regularized control and probabilistic inference.}
The exponential trajectory tilt used as our reference construction is
inherited from linearly solvable and KL-regularized control, trajectory
optimization by approximate inference, and the broader control-as-inference
formulation~\citep{todorov2006linearly,toussaint2009robot,
levine2018reinforcement}. For language-model distillation,
\citet{zimmer2025rethinking} optimize task reward subject to a bound on
divergence from a teacher. We use the trajectory tilt to define a
policy-relative safe-completion
target, then study the support--magnitude structure and selective
finite-candidate losses that arise when a frozen \vla{} decoder approximates
its next-block marginal.

\paragraph{Vision-language-action policies and safety alignment.}

Generalist robot policies increasingly express control as conditional
generation over action tokens, chunks, or continuous trajectories.
\citet{brohan2022rt1,brohan2023rt2} learn observation- and
instruction-conditioned policies at scale, while \citet{driess2023palme}
incorporate embodied observations into a multimodal language model.
Open generalist policy families~\citep{kim2024openvla,octo2024}, flow-based
action generation~\citep{black2024pi0}, and efficient action
tokenization~\citep{pertsch2025fast} further broaden this foundation.
SafeVLA addresses safety through constrained policy optimization and
introduces Safety-CHORES as an embodied evaluation
environment~\citep{zhang2025safevla}. Our setting is complementary:
task-optimized and safety-aligned checkpoints remain fixed, and the
intervention acts only on candidate selection at inference time.

\paragraph{Constrained and future-aware decoding.}

Inference-time guidance changes a generative model's selection rule without
updating its parameters. Future discriminators and lookahead search evaluate
downstream properties of candidate text
continuations~\citep{yang2021fudge,lu2022neurologicAstar}.
\citet{ji2025inferenceguard} formulate safe language generation as constrained
sequential inference with learned task and safety critics, while
\citet{nguyen2026apps} use future-value estimates to guide particle resampling
toward a sequence-level power target. In robotics, SafeDec applies
temporal-logic constraints to predicted traces through hard masking and
Robustness Constrained Decoding~\citep{kapoor2025safedec}. We study the
history-conditioned mass of policy-supported suffixes that achieve safe
task completion after the candidate-induced transition. The deployed score
is a selective heuristic approximation to that target.

\paragraph{Planning, viability, and runtime assurance.}
Model-based control and Monte Carlo planning compare actions through predicted
future trajectories. \citet{huang2025treeopo} use a teacher's Monte
Carlo search trees to construct a curriculum of prefixes and estimate
prefix-aware advantages for policy training. Safe reinforcement learning
represents long-horizon constraints through budgets, augmented states, or
constrained objectives~\citep{sootla2022saute,lin2023trebi}. Viability theory and
Hamilton--Jacobi reachability characterize states from which admissible
evolution remains possible~\citep{aubin1991viability,bansal2017hamilton},
while runtime shields and predictive safety filters intervene on nominal
policies to enforce specified constraints~\citep{alshiekh2018safe,
hsu2021safety,wabersich2023data}. Our support notion is narrower by construction:
it is defined under the frozen continuation law rather than over all
admissible controls. Correspondingly, the deployed decoder neither
reconstructs a full belief state nor claims invariant-set or closed-loop
guarantees.

\paragraph{Evaluation of embodied safety.}

Safety-oriented VLA evaluation must distinguish task completion from
trajectory safety. Safety-CHORES evaluates embodied tasks with explicit
safety constraints~\citep{zhang2025safevla}. \citet{fan2026safevlabench} report
success-but-unsafe behavior and violation severity to expose unsafe
trajectories hidden by aggregate success. A converse ambiguity arises when
cumulative cost falls because a method terminates earlier or fails sooner.
We therefore report success counts, cumulative cost, violation counts, and
episode length jointly, and do not interpret lower unconditioned cost as
dominance when completion and exposure also decrease.

\end{document}